%% file: main.tex
\documentclass[11pt]{article}

\RequirePackage[OT1]{fontenc}
\RequirePackage{amsthm,amsmath,amssymb,amsfonts}
\RequirePackage[authoryear, round]{natbib}
\RequirePackage[colorlinks,citecolor=blue,urlcolor=blue]{hyperref}
\usepackage{bm, bbm, graphicx, float}
\usepackage{paralist}
\usepackage{bigstrut}
\usepackage{verbatim}
\usepackage{fullpage}

\usepackage{enumitem}

\numberwithin{equation}{section}
\theoremstyle{plain} 
\newtheorem{theorem}{Theorem}[section]

\newtheorem{lemma}[theorem]{Lemma}
                   
\newtheorem{claim}{Claim}

\newcommand{\proja}{\bm{\mathcal{P}_A}}
\newcommand{\projw}{\bm{\mathcal{P}_W}}

\newcommand{\mf}{{\mathcal{F}}}

\newcommand{\ts}{\bm{\tilde{S}}}

\newcommand{\estimator}{\bm{\hat{X}_{\tau_n}^*}}
\newcommand{\mlzero}{{M_{\lambda_0}}}
\newcommand{\mlone}{{M_{\lambda_1}}}
\newcommand{\mi}{\bm{\mathcal{I}}}
\newcommand{\mc}{{\mathcal{C}}}

\newcommand{\lossfunc}{{\mathcal{L}}}

\newcommand{\A}{{\bm{A}}}
\newcommand{\B}{{\bm{B}}}
\newcommand{\D}{{\bm{D}}}
\newcommand{\ee}{{\bm{e}}}
\newcommand{\G}{{\bm{G}}}
\newcommand{\eicH}{{\bm{H}}}
\newcommand{\I}{{\bm{I}}}
\newcommand{\Laplacian}{{\bm{L}}}
\newcommand{\Oh}{{\bm{O}}}
\newcommand{\Q}{{\bm{Q}}}
\newcommand{\eS}{{\bm{S}}}
\newcommand{\es}{{\bm{s}}}
\newcommand{\T}{{\bm{T}}}
\newcommand{\U}{{\bm{U}}}
\newcommand{\V}{{\bm{V}}}
\newcommand{\W}{{\bm{W}}}
\newcommand{\X}{{\bm{X}}}
\newcommand{\x}{{\bm{x}}}
\newcommand{\Y}{{\bm{Y}}}
\newcommand{\Z}{{\bm{Z}}}
\newcommand{\z}{{\bm{z}}}
\newcommand{\BoldTheta}{{\bm{\Theta}}}
\newcommand{\BoldGamma}{{\bm{\Gamma}}}

\DeclareMathOperator{\diag}{diag}

\def\twoImages#1#2#3#4#5#6 
{
  \centerline{\hfill
    \includegraphics[width=#2]{#1}
    \hfill
    \includegraphics[width=#5]{#4}
    \hfill
    }
  \centerline{\hfill\makebox[#2]{#3}\hfill\makebox[#5]{#6}\hfill\medskip}
}

\begin{document}


\title{Detecting Overlapping Communities in Networks Using Spectral Methods}

\author{Yuan Zhang, Elizaveta Levina and Ji Zhu}
\date{}
\maketitle

\input{./abstract.tex}

\section{Introduction}
\input{./introduction.tex}\label{section_introduction}

\section{The overlapping continuous community assignment model}\label{section_generative_model}
\input{./generative_model.tex}

\section{A spectral algorithm for fitting the model}\label{section_estimation}
\input{./estimation.tex}

\section{Asymptotic consistency}\label{section_consistency}
\input{./consistency.tex}
\section{Evaluation on synthetic networks}\label{section_simulations}
\input{./simulations.tex}
\section{Application to SNAP ego-networks}\label{section_data_applications}
\input{./data_applications.tex}

\section{Discussion}\label{section_discussions}
\input{./discussion}

\section{Appendix}\label{section_appendix}
\input{./appendix.tex}

\bibliographystyle{apalike}
\bibliography{main}

\end{document}

%% file: abstract.tex
\begin{abstract}
Community detection is a fundamental problem in network analysis.  In practice, communities often overlap, which makes the problem more challenging.
Here we propose a general, flexible, and interpretable generative model for overlapping communities, which can be viewed as generalizing several previous models in different ways.     We develop an efficient spectral algorithm for estimating the community memberships, which deals with the overlaps by employing the $K$-medians algorithm rather than the usual $K$-means for clustering in the spectral domain.  We show that the algorithm is asymptotically consistent when networks are not too sparse and the overlaps between communities not too large.   Numerical experiments on both simulated networks and many real social networks demonstrate that our method performs well compared to a number of benchmark methods for overlapping community detection. 
\end{abstract}

%% file: introduction.tex
The problem of community detection in networks has been actively studied in several distinct fields, including physics, computer science, statistics, and the social sciences.  Its applications include understanding social interactions of people \citep{Zachary1977, Add_health_resnick1997protecting} and animals \citep{Lusseau2003}, discovering functional regulatory networks of genes \citep{bolouri2010gene, Zhang:2009:PIN:1540618} and even designing parallel computing algorithms \citep{chamberlain1998graph, hendrickson2000graph}.  Community detection is in general a challenging task.   The challenges include defining what a community is (commonly taken to be a group of nodes that have more connections to each other than to the rest of the network, although other types of communities are not unusual), formulating realistic and tractable statistical models of networks with communities, and designing fast scalable algorithms for fitting such models.

In this paper, we focus on network models with overlapping communities, with nodes potentially belonging to more than one community at a time.  This is common in real-world networks   \citep{palla2005uncovering, pizzuti2009overlapped}, and yet most literature to date has focused on partitioning the network into non-overlapping communities, with some notable exceptions discussed below.  Our goal is to design an overlapping community model that is flexible, interpretable, and computationally feasible.   We will thus focus on models which can be fitted by spectral methods, one of the most scalable tools for fitting non-overlapping community models available to date.

We start with a brief review of relevant work in community detection for non-overlapping communities, which mainly falls into one of two broad categories:  algorithmic methods, based on optimizing some criterion reflecting desirable properties of a partition over all possible partitions (see \citet{Fortunato2010} for a review), and model fitting, where a generative model with communities is postulated for the network and its parameters are estimated from the observed adjacency matrix (see \citet{Goldenberg2010} for a review).  Perhaps the most popular and best studied generative model for community detection is the stochastic block model (SBM) \citep{HolLei81, Holland83}. The SBM views the $n \times n$ network adjacency matrix $\A$, defined by $A_{ij} =1$ if there is an edge between $i$ and $j$ and 0 otherwise, as a random graph with independent Bernoulli-distributed edges.  The Bernoulli probabilities for the edges depend on the node labels $c_i$ which take values in $\{1, \dots, K\}$, and the  $K \times K$ matrix $\B$ containing the probabilities of edges forming between different communities.   The node labels can be represented by an $n \times K$  binary community membership matrix $\Z$ with exactly one ``1'' in each row, $Z_{ik} = \mathbf{1}[c_i = k]$ for all $i$, $k$.    Then the probabilities of edges are given by $\W \equiv \mathbb{E}(\A) = \Z \B \Z^T$.   Thus in this model, a node's label determines its behavior entirely, and thus all nodes in the same community are ``stochastically equivalent'', and in particular have the same expected degree.     This is known to be often violated in practice, due to commonly present ``hub'' nodes with many more connections than other nodes in their community.    The degree-corrected stochastic block model (DCSBM) \citep{Karrer10} was proposed to address this limitation, which multiplies the probability of an edge between nodes $i$ and $j$ by the product of node-specific positive ``degree parameters'' $\theta_i \theta_j$.  Both SBM and DCSBM can be consistently estimated by maximizing the likelihood \citep{Bickel&Chen2009, Zhaoetal2012}, but directly optimizing the likelihood over all label assignments is not computationally feasible.   A number of faster algorithms for fitting these models have been proposed in recent years, including pseudo-likelihood \citep{amini2013pseudo}, belief propagation \citep{Decelle.et.al.2011}, spectral approximations to the likelihood \citep{Newman2013, Le&Levina&Vershynin2014}, and generic spectral clustering \citep{von2007tutorial}, used by many and analyzed, for example, in \citet{Rohe2011, sarkar2013role}, and \citet{lei2013consistency}.   It was further shown that regularization improves on spectral clustering substantially \citep{amini2013pseudo, chaudhuri2012spectral}, and its theoretical properties have been further analyzed by \citet{qin2013regularized} and \citet{joseph2013impact}.  While for specific likelihoods one can develop methods that are both fast and more accurate than spectral clustering, such as the pseudo-likelihood \citep{amini2013pseudo}, in general spectral methods remain the most scalable option available.  

While the majority of the existing models and algorithms for community detection focus on discovering non-overlapping communities, there has been a growing interest in exploring the overlapping scenario, although both extending the existing models to the overlapping case and developing brand new models remain challenging. Like methods for non-overlapping community detection, most existing approaches for detecting overlapping communities can be categorized as either algorithmic or model-based methods. Model-based methods focus on specifying how node community memberships determine edge probabilities. For example, the overlapping stochastic block model (OSBM) \citep{latouche2009overlapping} extends the SBM by allowing the entries of the membership matrix $\Z$ to be independent Bernoulli variables, thus allowing multiple ``1''s in one row, or all ``0''s.   The mixed membership stochastic block model \citep{Airoldi2008} draws membership vectors $\Z_{i\cdot}$ from a Dirichlet prior.  The membership vector is drawn again to generate every edge, instead of being fixed for the node, so the community membership for node $i$ varies depending on which node $j$ it is interacting with.  The ``colored edges'' model \citep{Ball&Karrer&Newman2011}, sometimes referred to as the Ball-Karrer-Newman model or BKN, allows continuous community membership by relaxing the binary $\Z$ to a matrix with non-negative entries (with some normalization constraints for identifiability), and discarding the matrix $\B$.  The Bayesian nonnegative matrix factorization model \citep{psorakis2011overlapping} is related to the  model but with notable differences.

Algorithmic methods for overlapping community detection mostly rely on local greedy searches and intuitive criteria. Current approaches include detecting each community separately by maximizing a local measure of goodness of the estimated community \citep{lancichinetti2011finding} and updating an initial estimate of the community membership by neighborhood vote \citep{gregory2010finding}.   Local methods typically rely heavily on a good starting value.   Global algorithmic approaches include computing a non-negative matrix factorization approximation to the adjacency matrix and extracting a binary membership matrix from one of the factors \citep{wang2011community, gillis2012fast}.  Many heuristic methods do not take heterogeneous node degrees into account, and we found empirically they can perform poorly in the presence of hubs (see Section \ref{section_simulations}).

In this paper, we propose a new generative model for overlapping communities, the overlapping continuous community assignment model (OCCAM).   It allows a node to belong to different communities to a different extent, via the membership vector $\Z_{i\cdot}$ with non-negative entries which represent how strongly a node is associated with various communities.   We also allow arbitrary degree distributions in a manner similar to the DCSBM, and retain the $K \times K$ matrix $\B$ which allows to interpret connections between communities and compare them.  All the model parameters (membership vectors, degree corrections, and community-level connectivity) are identifiable under certain constraints which we will state explicitly.      We also develop a fast spectral algorithm to fit OCCAM.  Typically, spectral clustering projects the adjacency matrix or its Laplacian onto the $K$ leading eigenvectors representing the nodes' latent positions, and performs $K$-means in that lower-dimensional space to estimate community memberships.   Our key insight here is that when the nodes come from a mixture of clusters (as they would with multiple community memberships), $K$-means has no chance of recovering the cluster centers correctly;  but as long as there are enough pure nodes in each community, $K$-medians will still be able to identify the cluster centers correctly by ignoring the ``mixed'' nodes on the boundaries.   We show that our method produces asymptotically consistent parameter estimates as the number of nodes grows as long as there are enough pure nodes and the network is not too sparse.  We also employ a simple regularization scheme, since it is by now well known that regularizing spectral clustering substantially improves its performance, especially in sparse networks \citep{chaudhuri2012spectral, amini2013pseudo, qin2013regularized}.  We provide an explicit rate for the regularization parameter, implied by our consistency analysis, and show that the overall performance is robust to the choice of the constant multiplier in the regularization parameter as long as the rate is specified correctly.

The rest of the paper is organized as follows.  We introduce the model and discuss parameter identifiability in Section \ref{section_generative_model}, present the two-stage spectral clustering algorithm in Section \ref{section_estimation}, and state consistency results and describe the choice of the regularization parameter in Section \ref{section_consistency}.  Some simulation results are presented  in Section \ref{section_simulations}, where we investigate robustness of our method to the choice of regularization parameter and compare it to a number of benchmark methods for overlapping community detection.  We apply the proposed method to a large number of real social ego-networks (networks consisting of all friends of one or several users) from Facebook, Twitter, and GooglePlus in Section \ref{section_data_applications}. Section \ref{section_discussions} concludes the paper with a brief discussion of contributions, limitations, and future work.  All proofs are given in the supplemental materials.

%% file: generative_model.tex
\subsection{The  model}
Recall that we represent the network by its $n \times n$ adjacency matrix $\A$, a binary symmetric matrix with $\{A_{ij}, i < j\}$ independent Bernoulli variables and $\W \equiv \mathbb{E}(\A)$.  We will assume that $\W$ has the form 
\begin{equation}
 \W = \alpha_n {\BoldTheta} \Z \B \Z^T {\BoldTheta} \label{our_model_W} \ . 
\end{equation}
We call this formulation the Overlapping Continuous Community Assignment Model (OCCAM).   The factor $\alpha_n$ is a global scaling factor that controls the overall edge probability, and the only component that depends on $n$.   As is commonly done in the literature, for theoretical analysis we will let $\alpha_n \to 0$ at a certain rate, otherwise the network becomes completely dense as $n \to \infty$.   The $n\times n$ diagonal matrix ${\BoldTheta} =\diag(\theta_1, \ldots, \theta_n)$ contains non-negative degree correction terms that allow for heterogeneity in the node degrees, in the same fashion as under the DCSBM.   We will later assume that $\theta_i$'s are generated from a fixed distribution $\mf_{\Theta}$ which does not depend on $n$.   The $n \times K$ community membership matrix $\Z$ is the primary parameter of interest; the $i$-th row  $\Z_{i\cdot}$  represents node $i$'s propensities towards each of the $K$ communities.  We assume $Z_{ik} \geq 0$ for all $i$, $k$, and $\|\Z_{i\cdot}\|_2=1$ for identifiability.  Formally, a node is ``pure'' if $Z_{ik}=1$ for some $k$.   Later, we will also assume that the rows $\Z_{i\cdot}$'s are generated independently from a fixed distribution $\mf_{Z}$ that does not depend on $n$.
Finally,  the $K \times K$ matrix  $\B$ represents (scaled) probabilities of connections between pure nodes of all communities.   Since we are already using $\alpha_n$ and ${\BoldTheta}$, we constrain all diagonal elements of $\B$ to be 1 for identifiability.    Other constraints are also needed to make the model fully identifiable;  we will discuss them in Section \ref{section_identifiability}.

Note that the general form \eqref{our_model_W} can, with additional constraints, incorporate many of the other previously proposed models as special cases.   If all nodes are pure and $\Z$ has exactly one ``1'' in each row, we get DCSBM;  if we further assume all $\theta_i$'s are equal, we have the regular SBM.  If the constraint  $\|\Z_{i\cdot}\|_2=1$ is removed and the entries of $\Z$ are required to be 0 or 1, and all $\theta_i$'s are equal, we have the OSBM of \citet{latouche2009overlapping}.   Alternatively, if we set $\B = \I$, we have the ``colored edges'' model of \citet{Ball&Karrer&Newman2011}.    
Our model is also related to the random dot product model (RDPM) \citep{nickel2007random, young2007random}, which stipulates $\W=\X_0\X_0^T$ for some (usually low-rank) $\X_0$.   This is true for our model
if $\B$ is semi-positive definite, since then  we can uniquely define  $\X_0=\sqrt{\alpha_n}{\BoldTheta} \Z \B^{1/2}$.   OCCAM is thus more general than all of these models, and yet is fully identifiable and interpretable.

\subsection{Identifiability}\label{section_identifiability}
The parameters in \eqref{our_model_W} obviously need to be constrained to guarantee identifiability of the model.   All models with communities, including the SBM, are considered identifiable if they are identifiable up to a permutation of community labels.    To show the  interplay between the model parameters, we first state identifiability conditions treating all of $\alpha_n$, ${\BoldTheta}$, $\Z$, and $\B$ as constant parameters, and then discuss what happens if ${\BoldTheta}$ and $\Z$ are treated as random variables as we do in the asymptotic analysis.  The following conditions are sufficient for identifiability: 
\begin{enumerate}[label=I\arabic{*}]
\item  \label{condition_iden_B} $\B$ is full rank and strictly positive definite, with $B_{kk} = 1$ for all $k$.    
\item \label{condition_iden_Z} All $Z_{ik} \geq 0$, $\|\Z_{i\cdot}\|_2=1$ for all $i = 1, \dots, n$, and there is at least one ``pure'' node in every community, i.e., for each $k=1,\ldots, K$, there exists at least one $i$ such that $Z_{ik} = 1$.   
\item  \label{condition_iden_theta}  The degree parameters $\theta_1, \ldots, \theta_n$ are all positive and $n^{-1} \sum_{i=1}^n \theta_i = 1$. \end{enumerate}

\begin{theorem}\label{theorem_identifiability}
 If conditions \eqref{condition_iden_B}, \eqref{condition_iden_Z} and \eqref{condition_iden_theta} hold, the model is identifiable, i.e., if a given probability matrix $\W$ corresponds to a set of parameters $(\alpha_n, {\BoldTheta}, \Z, \B)$ through \eqref{our_model_W}, these parameters are unique up to a permutation of community labels.
\end{theorem}

The proof of Theorem \ref{theorem_identifiability} is given in the supplemental materials.  In general, identifiability is non-trivial to establish for most overlapping community models,  since, roughly speaking, an edge between two nodes can be explained by either their common memberships in many of the same communities, or the high probability of edges between their two different communities, a problem that does not occur in the non-overlapping case.  
Among previously proposed models,  the OSBM was shown to be identifiable \citep{latouche2009overlapping}, but their argument does not extend to our model since they only considered $\Z$ with binary entries. The identifiability of the BKN model was not discussed by  \citet{Ball&Karrer&Newman2011}, but it is relatively straightforward (though still non-trivial) to show that it is identifiable as long as there are pure nodes in each community.

While Theorem \ref{theorem_identifiability}  makes the model in \eqref{our_model_W} well defined, it is also common practice in the community detection literature to treat some of the model components as random quantities. For example,  \citet{Holland83}  treat community labels under the SBM as sampled from a multinomial distribution, and \citet{Zhaoetal2012} treat the degree parameters $\theta_i$'s in DCSBM as sampled from a general discrete distribution.   
For our consistency analysis, treating $\theta_i$'s and $\Z_{i\cdot}$'s as random significantly simplifies conditions and allows for an explicit choice of rate for the tuning parameter $\tau_n$, which will be defined in Section \ref{section_estimation}.     We will thus treat ${\BoldTheta}$ and $\Z$ as random and independent of each other for the purpose of theory, assuming that the rows of $\Z$ are independently generated from a distribution $\mf_{Z}$ on the unit sphere, and  $\theta_i$'s are i.i.d.\ from a distribution $\mf_\Theta$ on positive real numbers.    The conditions \ref{condition_iden_Z} and \ref{condition_iden_theta} are then replaced with the following two conditions, respectively:

\begin{enumerate}[label=RI\arabic{*}]
\setcounter{enumi}{1}
\item  \label{condition_iden_Z_r}    $\mf_Z = \pi_p\mf_p + \pi_o\mf_o$ is a mixture of a multinomial distribution $\mf_p$ on $K$ categories for pure nodes and an arbitrary distribution  $\mf_o$ on $\{\z\in\mathbb{R}^K: \z_k \geq 0, \|\z\|_2=1\}$ for nodes in the overlaps, and $\pi_p > \epsilon > 0$.
	\item \label{condition_iden_theta_r} $\mf_\Theta$ is a probability distribution on $(0,\infty)$ satisfying $\int_{0}^{\infty} t \, d\mf_\Theta(t) = 1$.
\end{enumerate}

The distribution $\mf_o$ can in principle be any distribution on the positive quadrant of the unit sphere.   For example, one could first specify that with probability $\pi_{k_1,\ldots,k_m}$, node $i$ belongs to communities $\{k_1,\ldots,k_m\}$, and then set $Z_{ik}=\frac{1}{\sqrt{m}} \mathbf{1}(k\in\{k_1,\ldots,k_m\})$.   Alternatively, one could generate values for the $m$ non-zero entries of $\Z_{i\cdot}$ from an $m$-dimensional Dirichlet distribution, and set the rest to 0.

%% file: estimation.tex
The primary goal of fitting this model is to estimate the membership matrix $\Z$ from the observed adjacency matrix $\A$, although other parameters may also be of interest.    Since computational scalability is one of our goals, we focus on algorithms based on spectral decompositions, one of the most scalable approaches available.      Recall that spectral clustering typically works by first representing all data points (the $n$ nodes) by an $n \times K$ matrix $\X$ consisting of leading eigenvectors of a matrix derived from the data, which we call $\G$ for now, and then applying $K$-means clustering to the rows of $\X$.  For example, under the SBM,  the matrix $\G$ should be chosen to have eigenvectors $\X$ that approximate the eigenvectors $\X_0$ of $\W = \mathbb{E}(\A)$ as closely as possible, since the eigenvectors of $\W$ are piecewise constant and contain all the community information.   A naive choice $\G = \A$ is intuitively appealing, though it has been shown in practice and in theory \citep{sarkar2013role} that the graph Laplacian of $\A$, i.e., $\Laplacian=\D^{-1/2}\A \D^{-1/2}$, where $\D=diag(\A\mathbbm{1})$, is a better choice, or, for sparse graphs, different regularized versions of $\Laplacian$ \citep{amini2013pseudo, chaudhuri2012spectral,  qin2013regularized, joseph2013impact}.    An additional step of normalizing the rows of $\X$ before performing $K$-means is often appropriate if the underlying model is assumed to be the degree-corrected stochastic blockmodel \citep{qin2013regularized}. 

Regardless of the matrix chosen to estimate the eigenvectors of $\W$, the key difference between the regular SBM under which spectral clustering is usually studied and our model is that under the SBM there are only $K$ unique rows in $\X_0$, and thus $K$-means can be expected to accurately cluster the rows of $\X$, which is a noisy version of $\X_0$.    Under our model, the rows of $\X_0$  are linear combinations of the ``pure'' rows corresponding to ``centers'' of the $K$ communities.   Thus even if we could recover $\X_0$ exactly, $K$-means is not expected to work, and it is in fact straightforward to show that the $K$-means algorithm does not recover the positions of pure nodes correctly unless non-pure nodes either vanish in proportion or converge to pure nodes' latent positions as $n$ grows (proof omitted here as it is not needed for our main argument).   The key idea of our algorithm is to replace $K$-means with $K$-medians clustering:   if the proportion of pure nodes is not too low, then the latent positions of the cluster centers can still be recovered correctly, and therefore the coefficients of mixed nodes can be estimated accurately by projecting onto the pure nodes.   Other details of the algorithm involve regularization and normalization that are necessary for dealing with sparse networks and heterogeneous degrees.

Our algorithm for fitting the OCCAM takes as input the adjacency matrix $\A$ and a regularization parameter $\tau_n >0$ which we use to regularize the estimated latent node positions directly.  This is easier to handle technically than regularizing the Laplacian, and we will give an explicit rate for $\tau_n$  that guarantees asymptotic consistency in Section \ref{section_consistency}.   The algorithm proceeds as follows:

\begin{enumerate}
\item Compute $\hat{\U}_A \hat{\Laplacian}_A  \hat{\U}_A^T$, where $\hat{\Laplacian}_A$ is the $K\times K$ diagonal matrix containing the $K$ leading eigenvalues of $\A$, and $\hat{\U}_A$ is the $n\times K$ matrix containing the corresponding eigenvectors.   While the true $\W = \mathbb{E}(\A)$ is positive definite, in practice some of the eigenvalues of $\A$ may be negative;  if that happens, we truncate them to 0.  Let $\hat{\X}\equiv \hat{\U}_A  \hat{\Laplacian}_A^{1/2}$ be the estimated latent node positions.
\item  Compute  $\hat{\X}^*$, a normalized and regularized version of $\hat{\X}$, the rows of which are given by  
$\hat{\X}^*_{i\cdot}=\frac{1}{\|\hat{\X}_{i\cdot}\|_2 + \tau_n} \hat{\X}_{i\cdot}$.
\item Perform $K$-medians clustering on the rows of $\hat{\X}^*$ and obtain $K$ estimated cluster centers $\es_1, \ldots, \es_K\in \mathbb{R}^K$, i.e., 
\begin{equation}
	 \{\es_1, \ldots, \es_K\} = \arg\min_{\es_1, \ldots, \es_K} \frac{1}{n}\sum_{i=1}^{n} \min_{\es\in\{\es_1, \ldots, \es_K\}} \Big\|\hat{\X}^*_{i\cdot} - \es\Big\|_2 \label{K_medians_X_hat}
\end{equation}
Form the $K \times K$ matrix  $\hat{\eS}$ with rows equal to the estimated cluster centers $\hat{\es}_1, \ldots, \es_K$.  
\item Project the rows of $\hat{\X}^*$ onto the span of $\es_1, \ldots, \es_K$, i.e.,  compute the matrix 
$\hat{\X}^* \hat{\eS}^{-1}$
and normalize its rows to have norm 1 to obtain the estimated community membership matrix $\hat{\Z}$.  
\end{enumerate}
This algorithm can also be used to obtain other types of community assignments.  For example, to obtain binary rather than continuous community membership, we can threshold each element of $\hat{\Z}$ to obtain $\hat{Z}^0_{ik} = \mathbf{1}(\hat{Z}_{ik}>\delta_K)$ (see Section \ref{section_simulations} and Section \ref{section_data_applications}).    To obtain assignments to non-overlapping communities, we can set  $\hat{c}_i =\arg\max_{1 \le k \le K}\hat{Z}_{ik}$.

%% file: consistency.tex
\subsection{Main result}
In this section, we show consistency of our algorithm for fitting the OCCAM as the number of nodes $n$ and possibly the number of communities $K$ increase.   For the theoretical analysis, we treat $\Z$ and $\BoldTheta$ as random variables, as was done by \citet{Zhaoetal2012}. We first state  regularity conditions on the model parameters.    

\begin{enumerate}[label=A\arabic*]
\item The distribution $\mf_\Theta$  is supported on $(0, M_\theta)$, and for all $\delta>0$ satisfies  $\delta^{-1} \int_0^{\delta} d \mf_\Theta(t) \leq C_\theta$, where $M_\theta > 0$ and $C_\theta>0$ are global constants. \label{condition_theta_distribution}
\item Let $\lambda_0$ and $\lambda_1$ be the smallest and the largest eigenvalues of $\mathbb{E}[\theta_i^2\Z_{i\cdot}^T \Z_{i\cdot} \B]$, respectively. 
Then there exist global constants $M_{\lambda_0}>0$ and $M_{\lambda_1}>0$ such that $K \lambda_0 \geq \mlzero$ and $\lambda_1\leq \mlone$.  \label{condition_eigen_ZB}
\item There exists a global constant $m_B>0$ such that $\lambda_{\min}(\B)\geq m_B$. \label{condition_eigen_B}
\end{enumerate}

A key ingredient of our algorithm is the $K$-medians clustering, and consistency of $K$-medians requires its own conditions on clusters being well separated in the appropriate metric.   The \emph{sample} loss function for $K$-medians is defined by 
$$\lossfunc_n(\Q; \eS) =\frac{1}{n}\sum_{i=1}^n\min_{1\leq k \leq K}\|\Q_{i\cdot} - \eS_{k\cdot}\|_2$$
where 
$\Q\in\mathbb{R}^{n\times K}$ is a matrix whose rows $\Q_{i\cdot}$  are vectors to be clustered, and $\eS\in \mathbb{R}^{K\times K}$ is a matrix whose rows $\eS_{k\cdot}$ are cluster centers.

Assuming the rows of $\Q$ are i.i.d. random vectors sampled from a distribution ${\cal G}$, we similarly define  the \emph{population} loss function for $K$-medians by
$$\lossfunc({\cal G}; \eS) =\int \min_{1\leq k \leq K}\|\x-\eS_{k\cdot}\|_2 d{\cal G}. $$
Finally we define the Hausdorff distance, which is used here to
measure the dissimilarity between two sets of cluster centers. Specifically, for $\eS, \T\in\mathbb{R}^{K\times K}$, let $D_H(\eS, \T) = \min_\sigma\max_k \|\eS_{k\cdot} - \T_{\sigma(k)\cdot}\|_2$, where $\sigma$ ranges over all permutations of $\{1,\ldots,K\}$.

Let $\mf$ denote the distribution of $\X_{i\cdot}^*= \|\X_{i\cdot}\|_2^{-1} \X_{i\cdot}= \|\Z_{i\cdot}\B^{1/2}\|_2^{-1}  \Z_{i\cdot}\B^{1/2}$, where $\X_{i\cdot}=\theta_i \Z_{i\cdot}\B^{1/2}$. Note that each $\X_{i\cdot}^*$ is a linear combination of the rows of $\B^{1/2}$. If the distribution $\mf$ of these linear combinations puts enough probability mass on the pure nodes (rows of $\B^{1/2}$), the rows of $\B^{1/2}$ will be recovered by $K$-medians clustering and then the $\Z_{i\cdot}$'s be recovered via projection. Bearing this in mind, we assume the following condition on $\mf$ holds: 
\begin{enumerate}[label=B]
\item Let $\eS_\mf = \arg\min_{\eS}\lossfunc(\mf; \eS)$ be the global minimizer of the population $K$-medians loss function $\lossfunc(\mf; \eS)$.  Then $\eS_\mf = \B^{1/2}$ up to a row permutation. Further, there exists a global constant $M$ such that, for all $\eS$, $\lossfunc(\mf; \eS) - \lossfunc(\mf; \eS_\mf) \geq   M K^{-1}  D_H(\eS, \eS_\mf)$. \label{condition_derivative}
\end{enumerate}

Condition \ref{condition_derivative} essentially states that the population $K$-medians loss function, which is determined by $\mf$, has a unique minimum at the right place and there is  curvature around the minimum.  

\begin{theorem}[Main theorem]\label{main theorem}
Assume that the identifiability conditions \ref{condition_iden_B},  \ref{condition_iden_Z_r}, \ref{condition_iden_theta_r} and regularity conditions \ref{condition_theta_distribution}-\ref{condition_eigen_B}, 
	 \ref{condition_derivative} hold.   If  $n^{1-\alpha_0}\alpha_n\to\infty$ for some $0<\alpha_0<1$,  
$K=O(\log n)$,  and the tuning parameter is set to
	\begin{equation}
		\tau_n = C_{\tau} \frac{\alpha_n^{0.2}K^{1.5}}{n^{0.3}} \label{recommended_tau_N}
	\end{equation}
	where $C_{\tau}$ is a constant, then the estimated community membership matrix $\hat{\Z}$ is consistent in the sense that
	\begin{equation}
		\mathbb{P}\left( \frac{1}{\sqrt{n}}\|\hat{\Z}-\Z\|_F \leq C(n^{1-\alpha_0}\alpha_n)^{-\frac{1}{5}} \right) \geq 1 - P(n, \alpha_n, K)
	\end{equation}
	where $C$ is a global constant, and $P(n, \alpha_n, K)\to 0$ as $n\to \infty$.
\end{theorem}

Remark: The condition $n^{1-\alpha_0}\alpha_n\to \infty$ is slightly stronger than $n \alpha_n\to \infty$, which was required for weak consistency of non-overlapping community detection with fixed $K$ using likelihood or modularities by \citet{Bickel&Chen2009}, \citet{Zhaoetal2012}, and others, and which is in fact necessary under the SBM \citep{mossel2014consistency}. The rate at which $K$ is allowed to grow works out to be $K=(n\alpha_n)^{\delta}$ for a small $\delta$ (see details in the supplemental materials), which is slower than the rates of $K$ allowed in previous work that considered a growing $K$ \citep{Rohe2011, Choietal2011}.  However, these results are not really comparable since we are facing additional challenges of overlapping communities and estimating a continuous rather than a binary membership matrix.  


\subsection{Example: checking conditions}
\label{subsection_discussion_of_conditions}

The planted partition model is a widely studied special case which we use to illustrate our conditions and their interpretation.   Let $\B = (1-\rho)\I_K + \rho {\bf 1} {\bf 1}^T$, $0\leq \rho < 1$, where $\I_K$ is the $K\times K$ identity matrix and ${\bf 1}$ is a column vector of all ones.  Then $\B^{1/2}$ is a $K\times K$ matrix with diagonal entires  $K^{-1}\left( \sqrt{(K-1)\rho+1} + (K-1)\sqrt{1-\rho} \right)$ and off-diagonal entires  $K^{-1}\left( \sqrt{(K-1)\rho+1} - \sqrt{1-\rho} \right)$.    We restrict the overlap to two communities at a time and generate the rows of the community membership  matrix $\Z$ by
\begin{equation}
\Z_{i\cdot} =	\begin{cases}
 \ee_k, \ 1 \le k \le  K & \textrm{ w. prob. } \pi^{(1)} \ , \\
 \frac{1}{\sqrt{2}}(\ee_k+\ee_l), \ 1 \le k < l \le K  & \textrm{ w. prob. } \pi^{(2)} \ ,
		\end{cases}
\end{equation}
where $\ee_k$ is a row vector that contains a one in the $k$th position and zeros elsewhere, and  $K\pi^{(1)} + \frac{1}{2} K(K-1) \pi^{(2)}=1$.  We set $\theta_i \equiv 1$ for all $i$, therefore conditions \ref{condition_iden_Z_r} and \ref{condition_iden_theta_r} hold.     

For a $K \times K$  matrix of the form $(a - b)\I_K + b \mathbf{1}\mathbf{1}^T$,  $a, b > 0$, the largest eigenvalue is $a +(K-1)b$ and all other eigenvalues are $a - b$.  Thus  $\lambda_{\max}(\B) = 1+(K-1)\rho$,  $\lambda_{\min}(\B) = 1-\rho$, and conditions \ref{condition_iden_B} and \ref{condition_eigen_B} hold.    To verify condition \ref{condition_eigen_ZB},  note $\mathbb{E}[\theta_i^2\Z_{i\cdot}^T \Z_{i\cdot} \B] =  \mathbb{E}[\Z_{i\cdot}^T \Z_{i\cdot}] \B$, and since 
\begin{equation*}
\Z_{i\cdot}^T \Z_{i\cdot} = \begin{cases}
				\ee_k^T \ee_k , \ 1\leq k\leq K & \textrm{ w. prob. } \pi^{(1)}  , \\
				\frac{1}{2}(\ee_k+\ee_l)^T(\ee_k+\ee_l) , \ 1\leq k<l\leq K & \textrm{ w. prob. } \pi^{(2)} , 
                         \end{cases}
\end{equation*}
we have $\mathbb{E}[\Z_{i\cdot}^T \Z_{i\cdot}] = \left( \pi^{(1)} + \frac{K-2}{2}\pi^{(2)} \right)\I_K + \frac{\pi^{(2)}}{2} \mathbf{1}\mathbf{1}^T$. Therefore,
\begin{align*}
	\lambda_{\max}(\mathbb{E}[\Z_{i\cdot}^T \Z_{i\cdot}]) &= \pi^{(1)}+(K-1)\pi^{(2)} \leq \frac{2}{K}\\
	\lambda_{\min}(\mathbb{E}[\Z_{i\cdot}^T \Z_{i\cdot}]) &= \pi^{(1)}+\frac{K-2}{2}\pi^{(2)} \geq \frac{1}{2K}
\end{align*}
Since  $\lambda_{\max}(\mathbb{E}[\Z_{i\cdot}^T \Z_{i\cdot}] \B) \leq \lambda_{\max}(\mathbb{E}[\Z_{i\cdot}^T \Z_{i\cdot}]) \lambda_{\max}(\B)$ and $\lambda_{\min}(\mathbb{E}[\Z_{i\cdot}^T \Z_{i\cdot}] \B) \geq \lambda_{\min}(\mathbb{E}[\Z_{i\cdot}^T \Z_{i\cdot}]) \lambda_{\min}(\B)$, condition \ref{condition_eigen_ZB} holds.   

It remains to check condition
\ref{condition_derivative}. Given $\x\in\mathbb{R}^{K}$ with $\|\x\|_2=1$, for any $\eS$, let $\es(\x)$ and $\es_{\mf}(\x)$ be the best approximations to $\x$ in $\ell_2$ norm among the rows of $\eS$ and $\eS_\mf$ respectively. Then we have
\begin{align}
	\lossfunc(\mf; \eS) - \lossfunc(\mf; \B^{1/2}) &= \left\{ \pi^{(1)}D_H(\eS, \B^{1/2}) + \int_{\x\neq (\B^{1/2})_{k\cdot}, 1\leq k \leq K} \|\x-\es(\x)\|_2 d\mf \right\}\nonumber\\
	&- \left\{ \int_{\x\neq (\B^{1/2})_{k\cdot}, 1\leq k \leq K} \|\x-\es_{\mf}(\x)\|_2 d\mf \right\}\nonumber\\
	&\geq \pi^{(1)}D_H(\eS, \B^{1/2}) - \int_{\x\neq (\B^{1/2})_{k\cdot}, 1\leq k \leq K} \|\es(\x) - \es_{\mf}(\x)\|_2 d\mf \nonumber\\
	&\geq \pi^{(1)}D_H(\eS, \B^{1/2}) - \int_{\x\neq (\B^{1/2})_{k\cdot}, 1\leq k \leq K} D_H(\eS, \B^{1/2}) d\mf \nonumber\\
	&= \left(\pi^{(1)} - \frac{K(K-1)}{2}\pi^{(2)}\right)D_H(\eS, \B^{1/2})\nonumber\\
	&= \left((K+1)\pi^{(1)}-1\right)D_H(\eS, \B^{1/2}) \label{SF_condition}
\end{align}
We then see that in order for \ref{condition_derivative} to hold, i.e., for the RHS of \eqref{SF_condition} to be non-negative and equal to zero only when $D_H(\eS, \eS_\mf)=0$, we need 
\begin{equation}
\pi^{(1)} > \frac{1}{K+1}  \left(1+\frac{M}{K}\right).
\end{equation}  
This gives a precise condition on the proportion of pure nodes for this example.  In general, the proportion of pure nodes cannot always be expressed explicitly other than through condition \ref{condition_derivative}.

%% file: simulations.tex
Our experiments on synthetic networks focus on two issues:  the choice of constant in the regularization parameter $\tau_n$, and comparisons of  OCCAM  to other overlapping community detection methods.  Since many other methods only output binary membership vectors, we use a performance measure based on binary overlapping membership vectors.  Following \citet{lancichinetti2009detecting}, we measure performance by an extended version of the \emph{normalized variation of information} (exNVI).   Consider two binary random vectors $\BoldGamma=(\Gamma_1,\ldots, \Gamma_K)$ and $\hat{\BoldGamma}=(\hat{\Gamma}_1, \ldots, \hat{\Gamma}_K)$, which indicate whether a node belongs to community $k$ in the true and estimated communities, respectively.   Define 
\begin{align}
 \bar{H}(\hat{\Gamma}_l|\Gamma_k) & = \frac{H(\hat{\Gamma}_l|\Gamma_k)}{H(\Gamma_k)} \ , \mbox{ where } \nonumber \\
  H(\Gamma_k) & = -\sum_z \mathbb{P}(\Gamma_k=z)\log \mathbb{P}(\Gamma_k=z), \nonumber \\ 
  H(\hat{\Gamma}_l|\Gamma_k) & =H(\Gamma_k, \hat{\Gamma}_l)-H(\Gamma_k), ~\mbox{and}~ \nonumber \\
  H(\Gamma_k, \hat{\Gamma}_l) & = -\sum_{z, \hat{z}} \mathbb{P}(\Gamma_k=z, \hat{\Gamma}_l=\hat{z})\log \mathbb{P}(\Gamma_k=z, \hat{\Gamma}_l=\hat{z}). \label{exNVI}
\end{align}
It can be seen that $\bar{H}(\hat{\Gamma}_l|\Gamma_k)$ takes values between $0$ and $1$, with 0 corresponding to $\hat{\Gamma}_l$ and $\Gamma_k$ being independent and $1$ to a perfect match. 
We then define the overall exNVI between $\BoldGamma$ and $\hat{\BoldGamma}$ to be
\begin{equation}
 \bar{H}(\BoldGamma, \hat{\BoldGamma}) = 1 - \min_\sigma \frac{1}{2 K}\sum_{k=1}^K\left[ \bar{H}(\hat{\Gamma}_{\sigma(k)}|\Gamma_k) + \bar{H}(\Gamma_k|\hat{\Gamma}_{\sigma(k)}) \right]
\end{equation}
where $\sigma$ ranges over all permutations on $\{1, \ldots, K\}$. We also define the sample versions of all the quantities in \eqref{exNVI} with probabilities replaced with frequencies, e.g., $\hat{H}(\Gamma_k) = - \sum_{z=0}^1 |\{i:\Gamma_{ik}=z\}|/n \cdot \log \left( |\{i:\Gamma_{ik}=z\}|/n \right)$, etc.

\subsection{Choice of constant for the regularization parameter}\label{subsection_choices_of_tuning_parameters}

The regularization parameter $\tau_n$ is defined by \eqref{recommended_tau_N}, up to a constant, as a function of $n$, $K$, and the unobserved $\alpha_n$.  
Absorbing a constant factor into $C_\tau$, we estimate $\alpha_n$ by 
\begin{equation}
	\hat{\alpha}_n=\frac{\sum_{i\neq j}A_{ij}}{n(n-1) K} \label{hat_alpha_n}
\end{equation}
and investigate the effect of the constant $C_\tau$ empirically.

For this simulation, we generate networks with $n = 500$ or  $2000$ nodes with $K=3$ communities. We consider two settings for $\theta_i$'s: (1) $\theta_i=1$ for all $i$ (no hubs), and (2)  $\mathbb{P}(\theta_i=1)=0.8$ and $\mathbb{P}(\theta_i=20)=0.2$ (20\% hub nodes). 
We generate $\Z$ as follows:  for $1\leq k_1< \ldots <k_m\leq K$, we assign $n \cdot \pi_{k_1\cdots k_m}$ nodes to the intersection of communities $k_1, \ldots, k_m$, and for each node $i$ in this set we set $Z_{ik}= m^{-1/2} \mathbf{1}(k \in \{k_1, \ldots, k_m\})$. 
Let $\pi_1=\pi_2=\pi_3=\pi^{(1)}$, $\pi_{12}=\pi_{13}=\pi_{23}=\pi^{(2)}$, $\pi_{123}=\pi^{(3)}$ and set $(\pi^{(1)}, \pi^{(2)}, \pi^{(3)}) = $ $(0.3, 0.03, 0.01)$.  Finally, we choose $\alpha_n$ so that the expected average node degree $\bar{d}$ is either 20 or 40.
We vary the constant factor $C_{\tau}$ in \eqref{recommended_tau_N} in the range $\{ 2^{-12}, 2^{-10}, \ldots, 2^{10}, 2^{12} \}$.   To use exNVI, we convert both the estimated $\hat{\Z}$ and $\Z$ to a binary overlapping  community assignment by thresholding its elements at $1/K$.
The results, shown in Figure \ref{figure_1}, indicate that the performance of OCCAM is stable over a  wide range of the constant factor ($2^{-12} - 2^5$), and degrades only for very large values of $C_{\tau}$. Based on this empirical evidence, we recommend setting 
\begin{equation}
	\tau_n = 0.1 \frac{\hat{\alpha}_n^{0.2}K^{1.5}}{n^{0.3}} \ . 
 \label{recommended_tau_n}
\end{equation}

\begin{figure}[htbp!]\label{figure_1}
	\centering
\twoImages{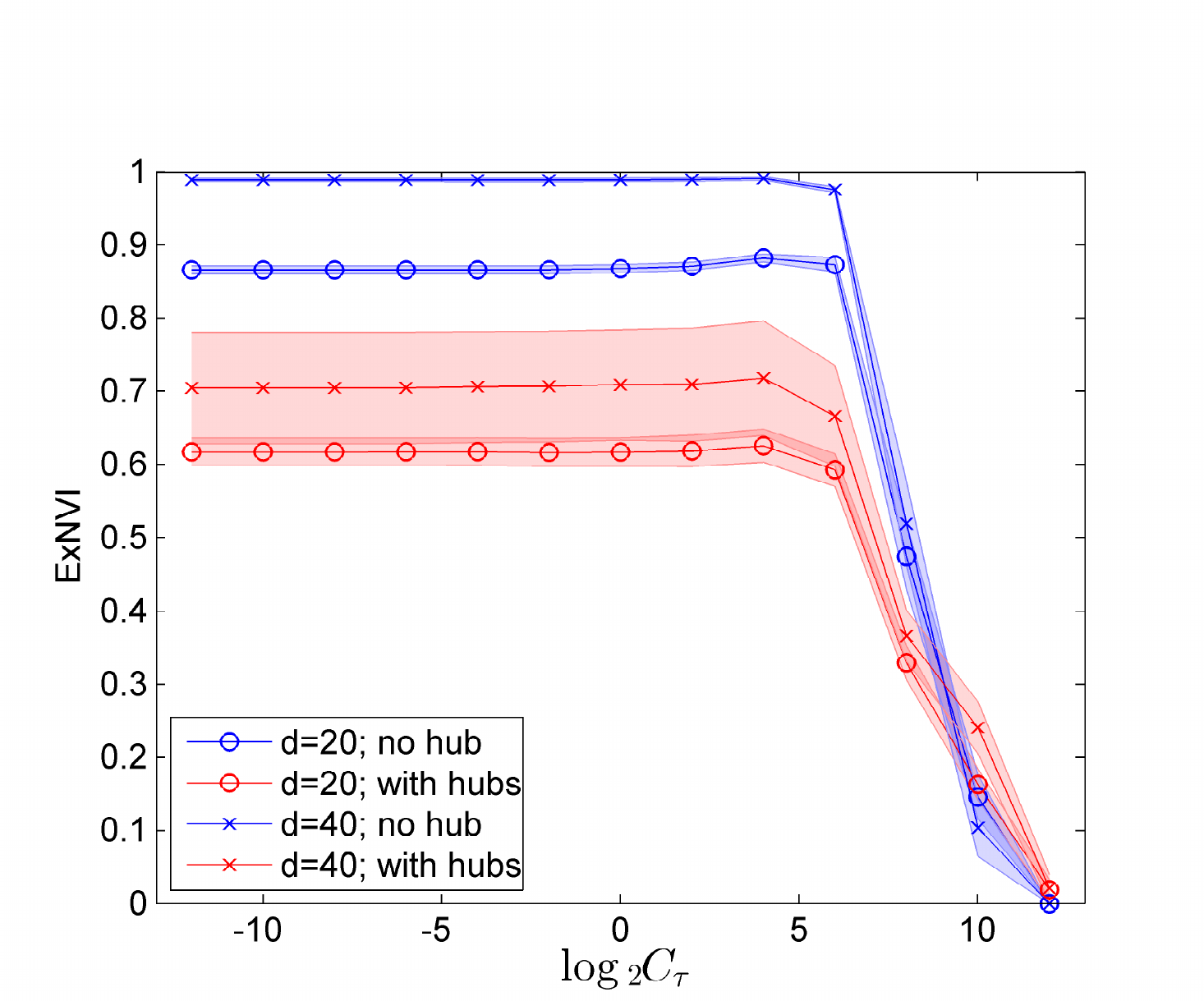}{0.4\textwidth}{(a) $\rho=0.1$, $n=500$}
		{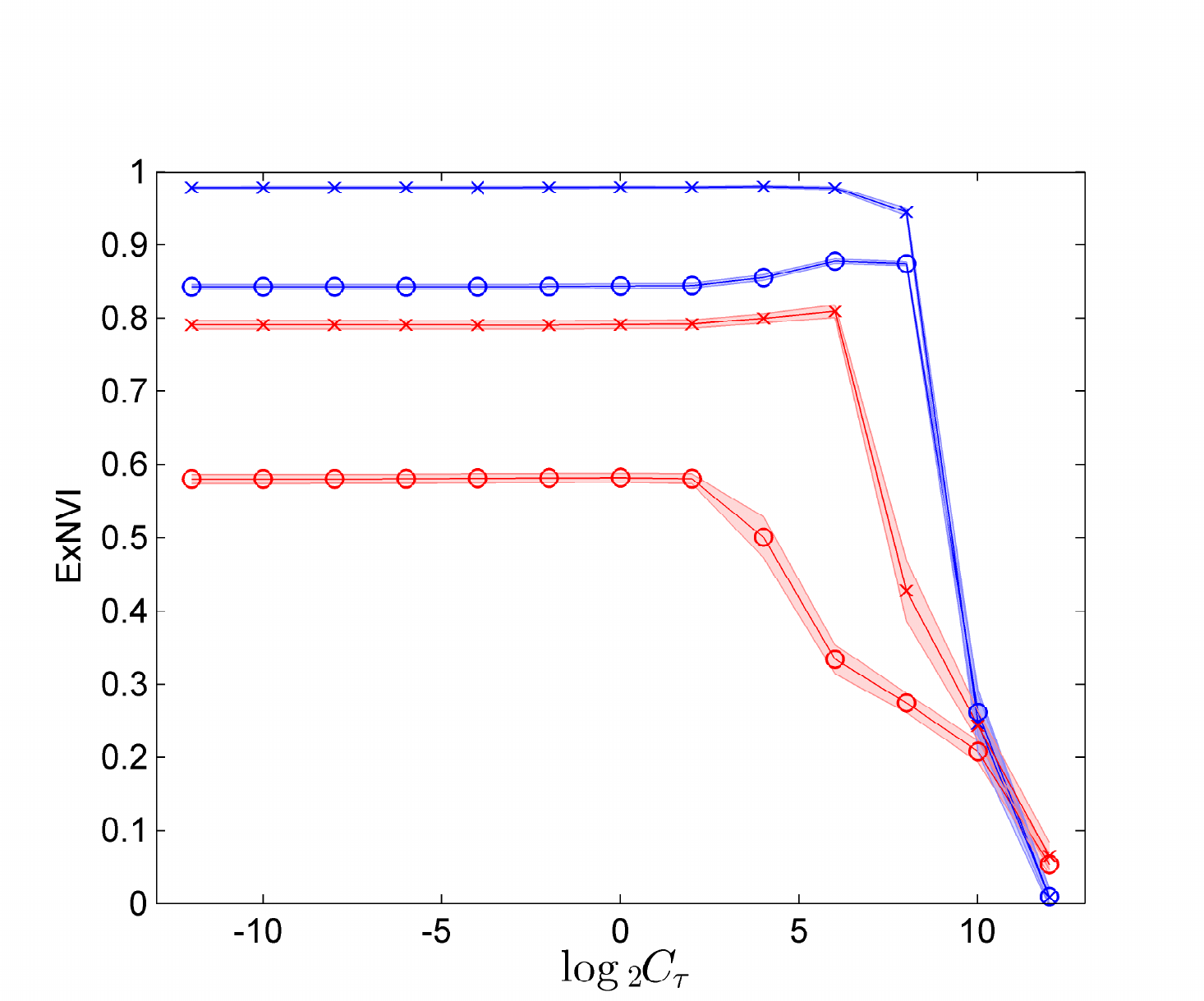}{0.4\textwidth}{(b) $\rho=0.1$, $n=2000$}
\twoImages{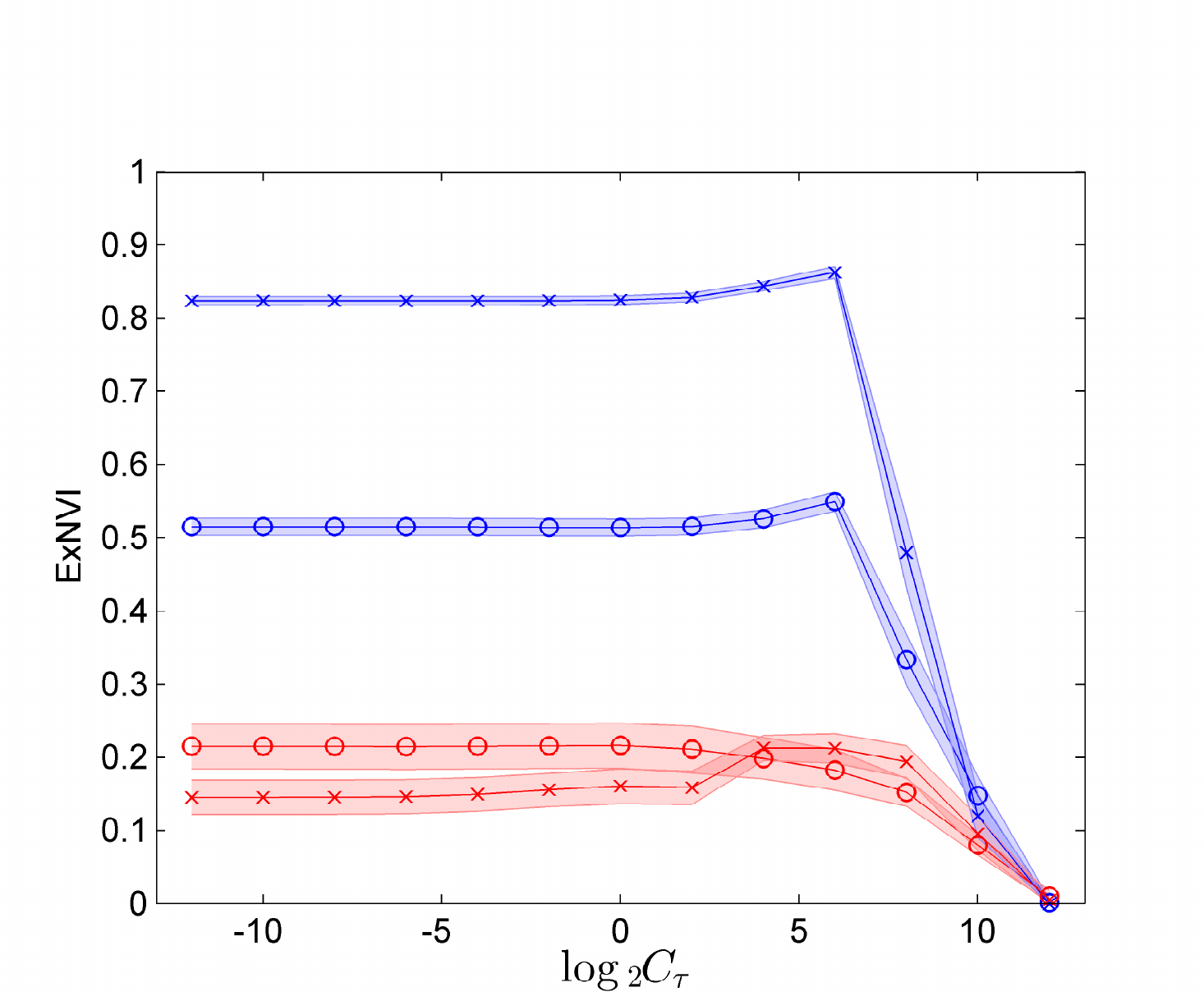}{0.4\textwidth}{(c) $\rho=0.25$, $n=500$}
		{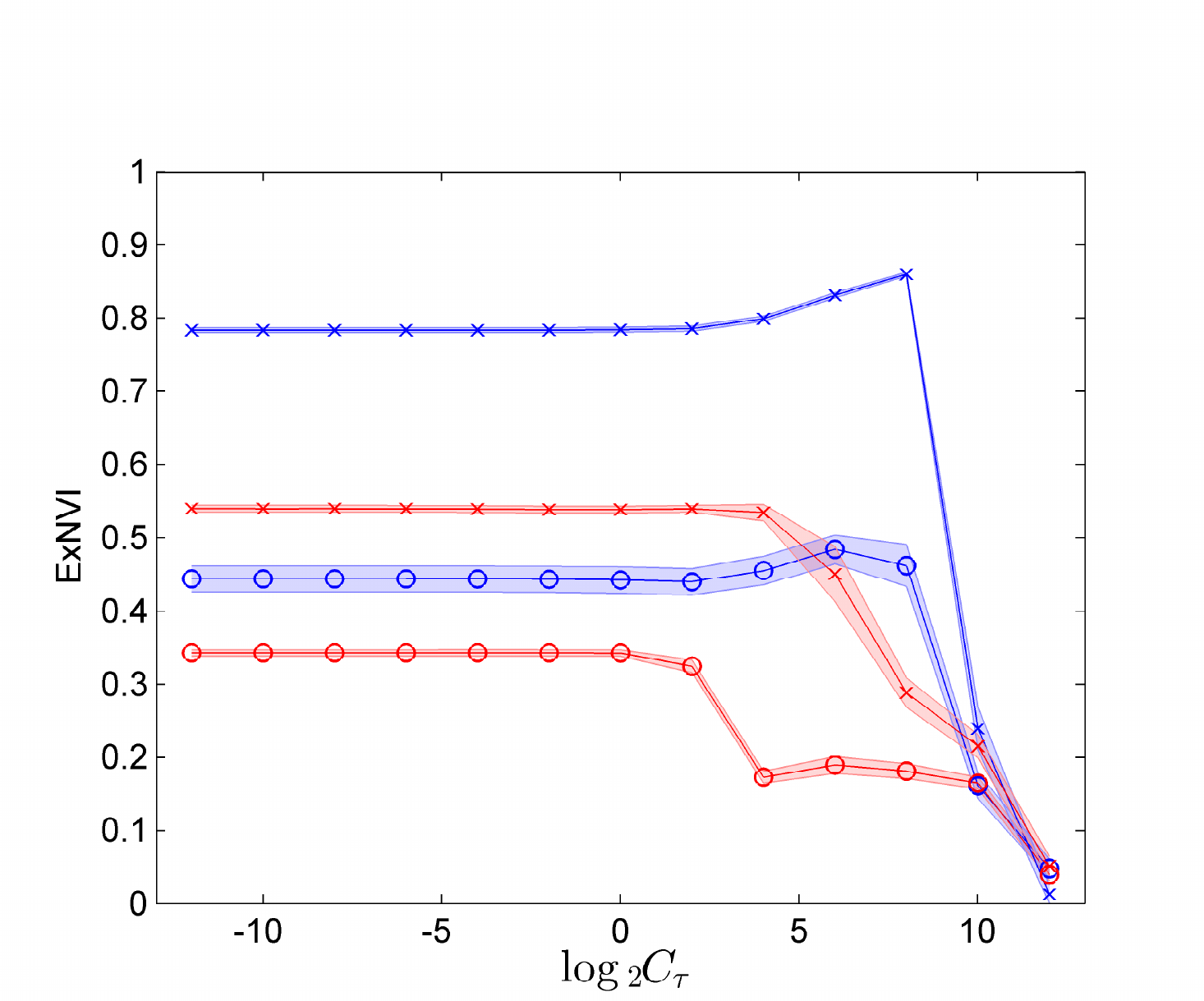}{0.4\textwidth}{(d) $\rho=0.25$, $n=2000$}
	\caption{Performance of OCCAM measured by exNVI as a function of $C_\tau$.}
\end{figure}

\subsection{Comparison to benchmark methods}
To compare OCCAM to other methods for overlapping community detection, we fix $n=500$ and use the same settings for $K$, $\Z$,  $\theta_i$'s and $\alpha_n$  as in Section \ref{subsection_choices_of_tuning_parameters}.  We set $B_{kk'} = \rho$ for $k\neq k'$, with $\rho = 0, 0.05, 0.10, \ldots, 0.5$, and set $(\pi^{(1)}, \pi^{(2)}, \pi^{(3)})$ to be either $(0.3, 0.03, 0.01)$ or $(0.25, 0.07, 0.04)$.   The regularization parameter $\tau_n$ is set to the recommended value  \eqref{recommended_tau_n}, and detection performance is measured by exNVI.



We compare OCCAM to both algorithmic methods and model-based methods that can be thought of as special cases of our model.  Algorithmic methods we compare include the order statistics local optimization method (OSLOM) by \citet{lancichinetti2011finding}, 
the community overlap propagation algorithm (COPRA) by \citet{gregory2010finding}, 
the nonnegative matrix factorization (NMF) on $\A$ computed via the algorithm of  \citet{gillis2012fast}, and the Bayesian nonnegative matrix factorization (BNMF) \citep{psorakis2011overlapping}.  Model-based methods we compare are two special cases of our model, the BKN overlapping community model  \citep{Ball&Karrer&Newman2011} and the overlapping stochastic blockmodel (OSBM) \citep{latouche2009overlapping}.  For methods that produce continuous community membership values, thresholding was applied for the purpose of comparisons.   For OCCAM and BNMF, where the membership vector is constrained to have norm 1, we use the threshold of $1/K$; for NMF, where there are no such constraints to guide the choice of threshold, we simply use a small positive number $10^{-3}$; and for BKN, we follow the scheme suggested by the authors and assign node $i$ to community $k$ if the estimated number of edges between $i$ and nodes in community $k$ is greater than $1$.  For each parameter configuration, we repeat the experiment 200 times.   Results are shown in Figure \ref{figure_2}.

As one might expect, all methods degrade as (1) the between-community edge probability approaches the within-community edge probability (i.e., $\rho$ increases); (2) the overlap between communities increases; and (3) the average node degree decreases. 
In all cases, OCCAM performs best, but we should also keep in mind that the networks were generated from the OCCAM model.  
BKN and BNMF perform well when $\rho$ is small but degrade much faster than OCCAM as $\rho$ increases, possibly because they require shared community memberships for nodes to be able to connect, thus eliminating connections between pure nodes from different communities;  NMF requires this too.   OSLOM detects communities by locally modifying initial estimates, and when $\rho$ increases beyond a certain threshold, the connections between pure nodes blur the ``boundaries'' between communities and lead OSLOM to assign all nodes to all communities. COPRA, a local voting algorithm, is highly sensitive to $\rho$ for the same reasons as OSLOM, and additionally suffers from numerical instability that sometimes prevents convergence.  OSBM performs well under the homogeneous node degree setting (when all $\theta_i=1$), where OSBM correctly specifies the data generating mechanism, but its performance degrades quickly in the presence of hubs.
Overall, in this set of simulations OCCAM has a clear advantage over its less flexible competitors.  



\begin{figure}[htbp!]\label{figure_2}
\centering
\vspace{-3em}
\twoImages{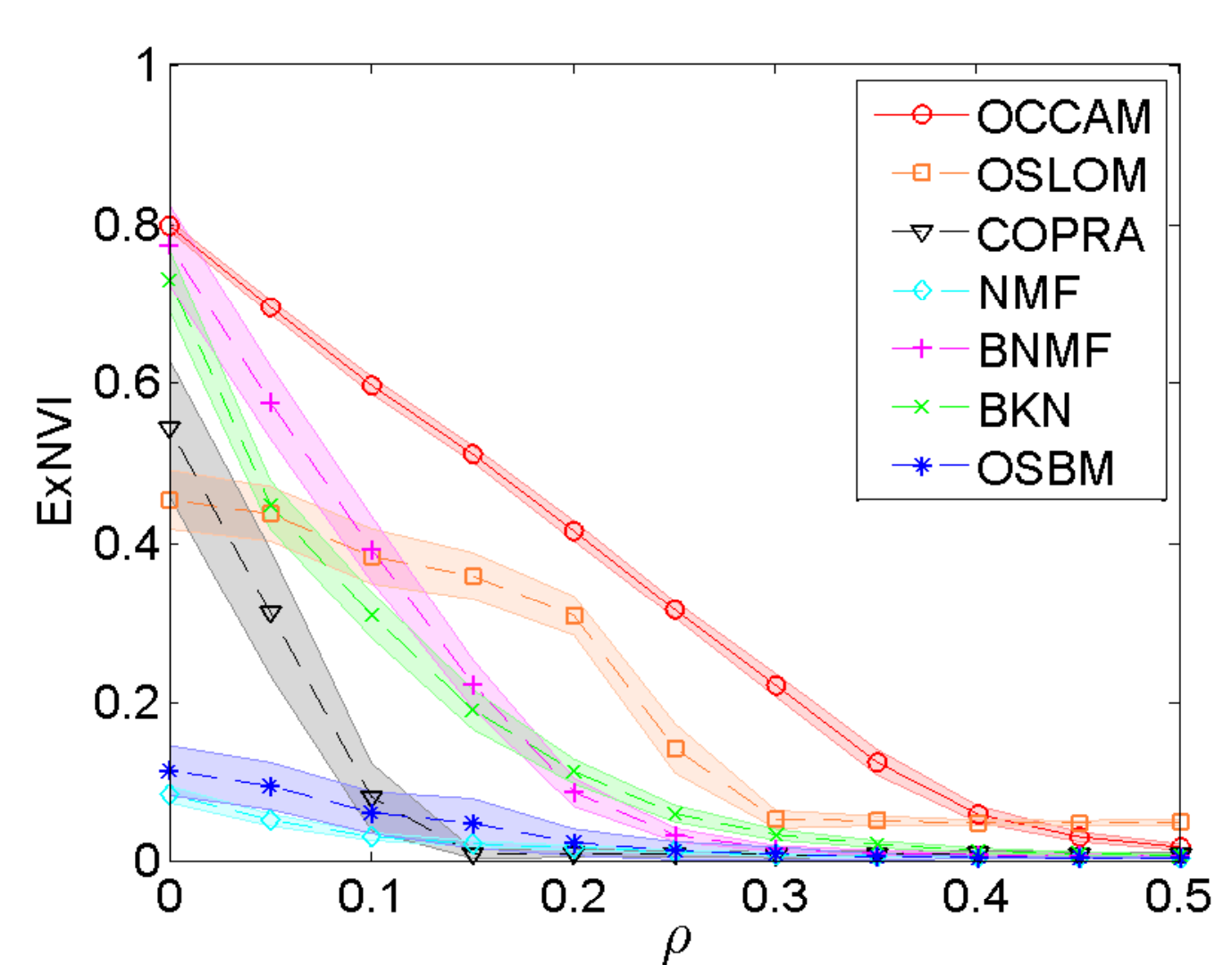}{0.4\textwidth}{(a) A, $d=20$, with hub nodes}  {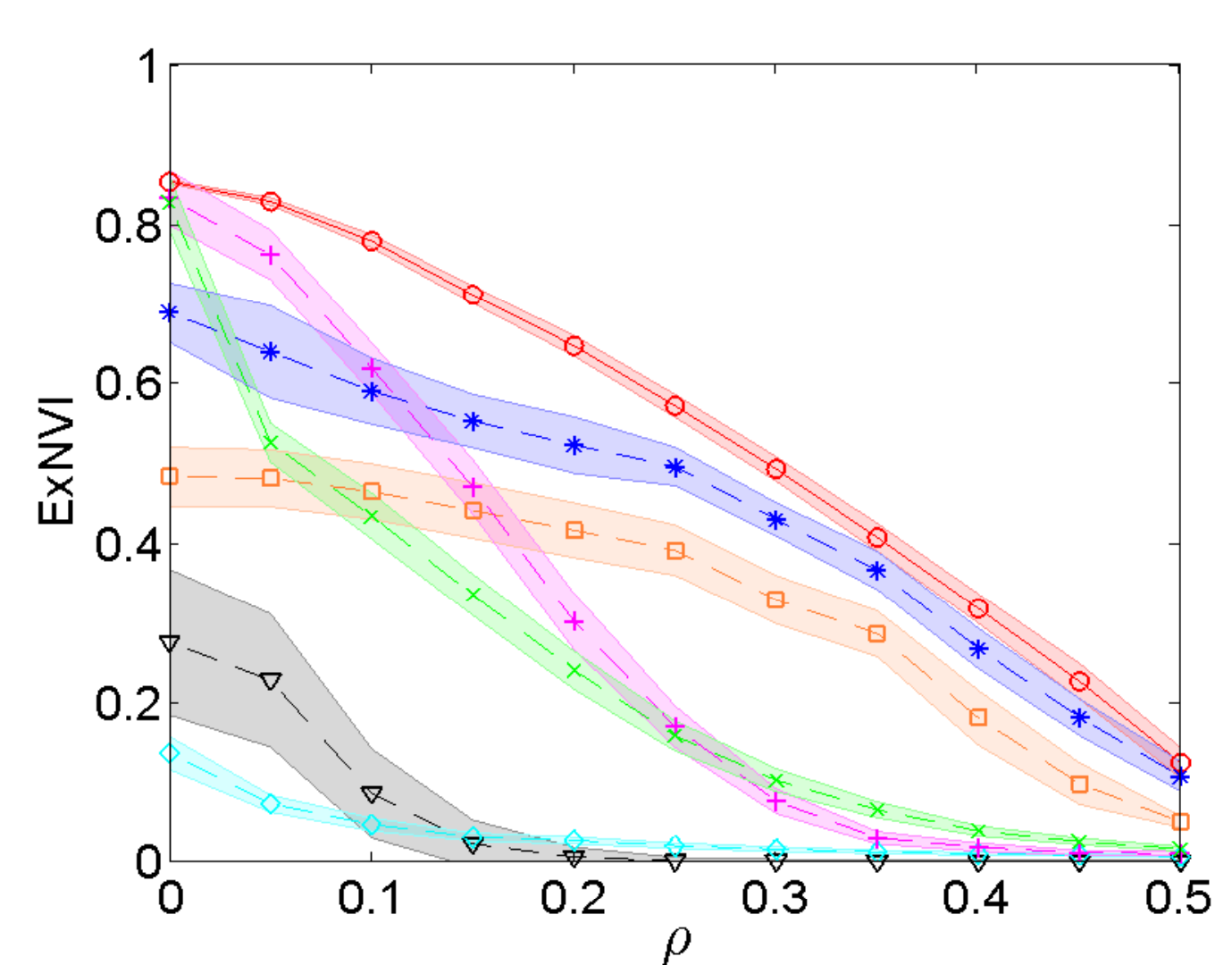}{0.4\textwidth}{(b) A, $d=40$, with hub nodes}
\twoImages{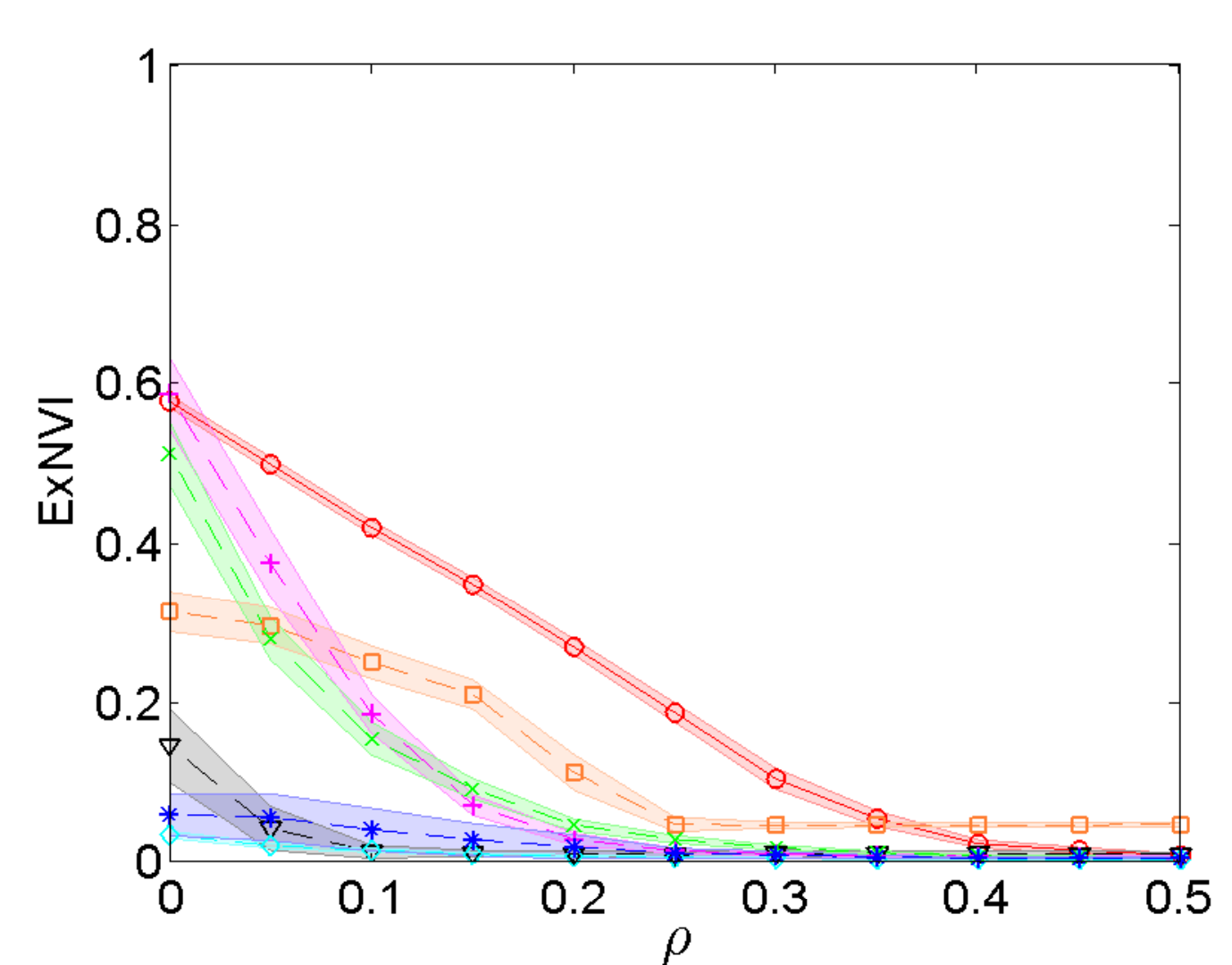}{0.4\textwidth}{(c) B, $d=20$, with hub nodes}  {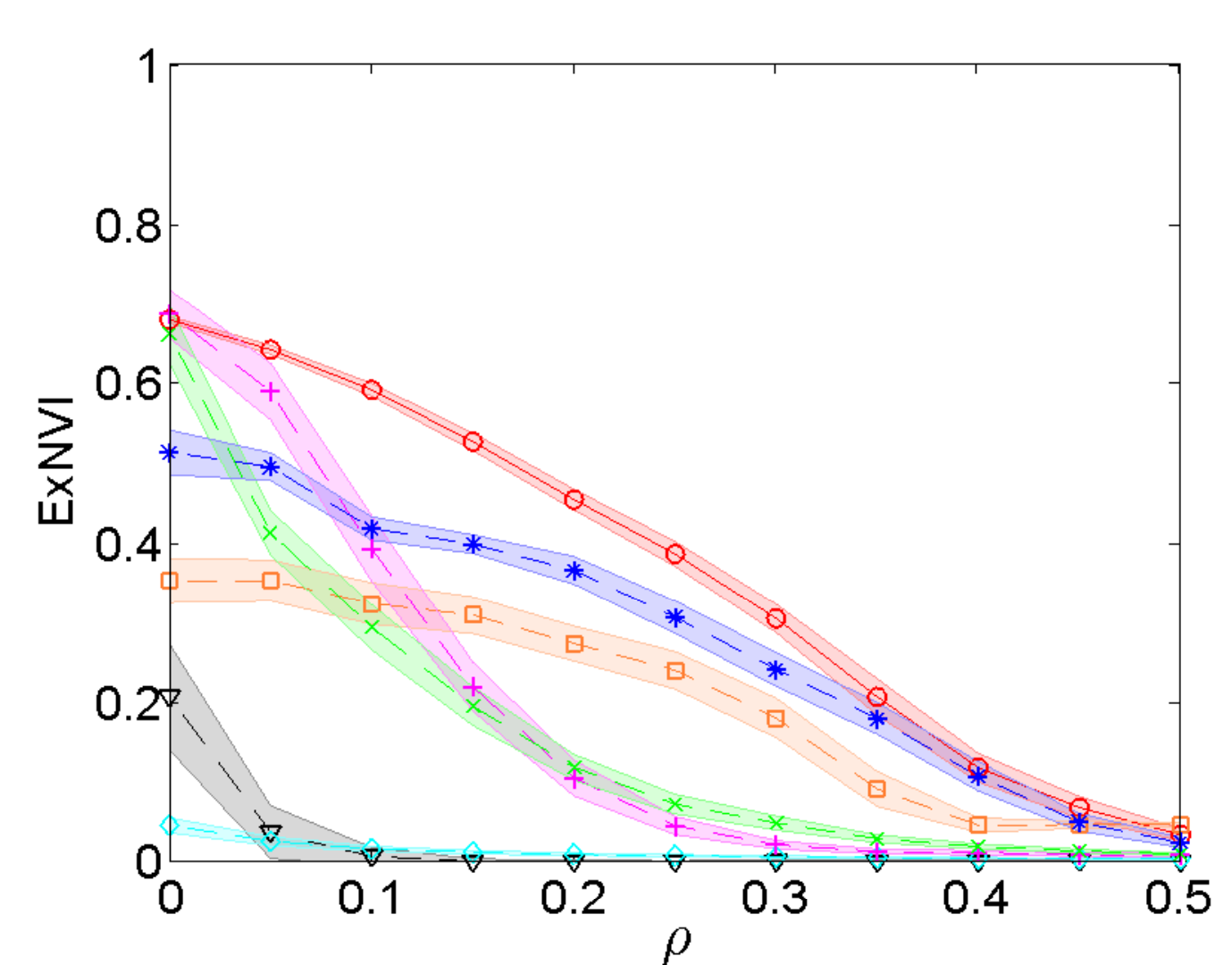}{0.4\textwidth}{(d) B, $d=40$, with hub nodes}
\twoImages{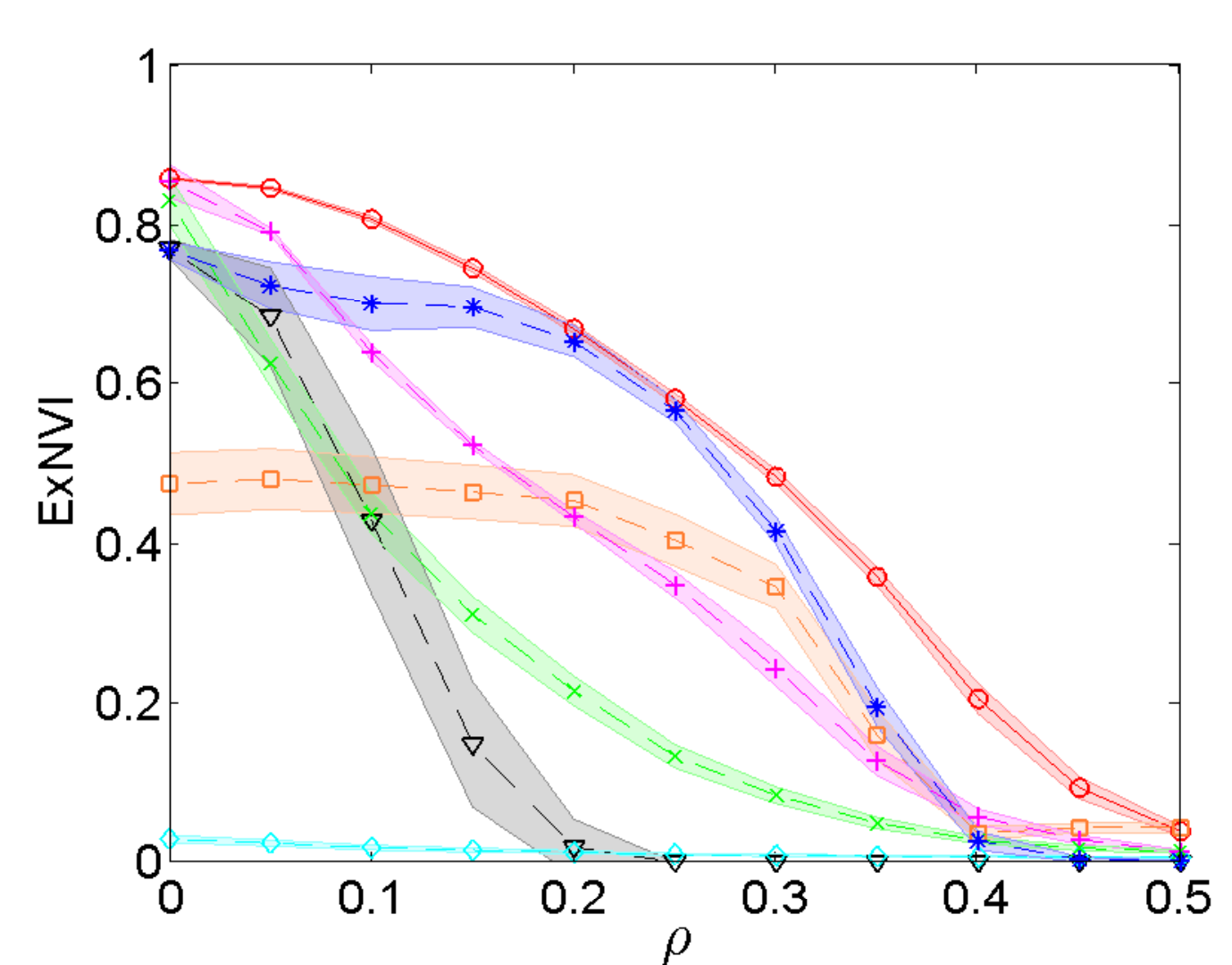}{0.4\textwidth}{(e) A, $d=20$, no hub node}  {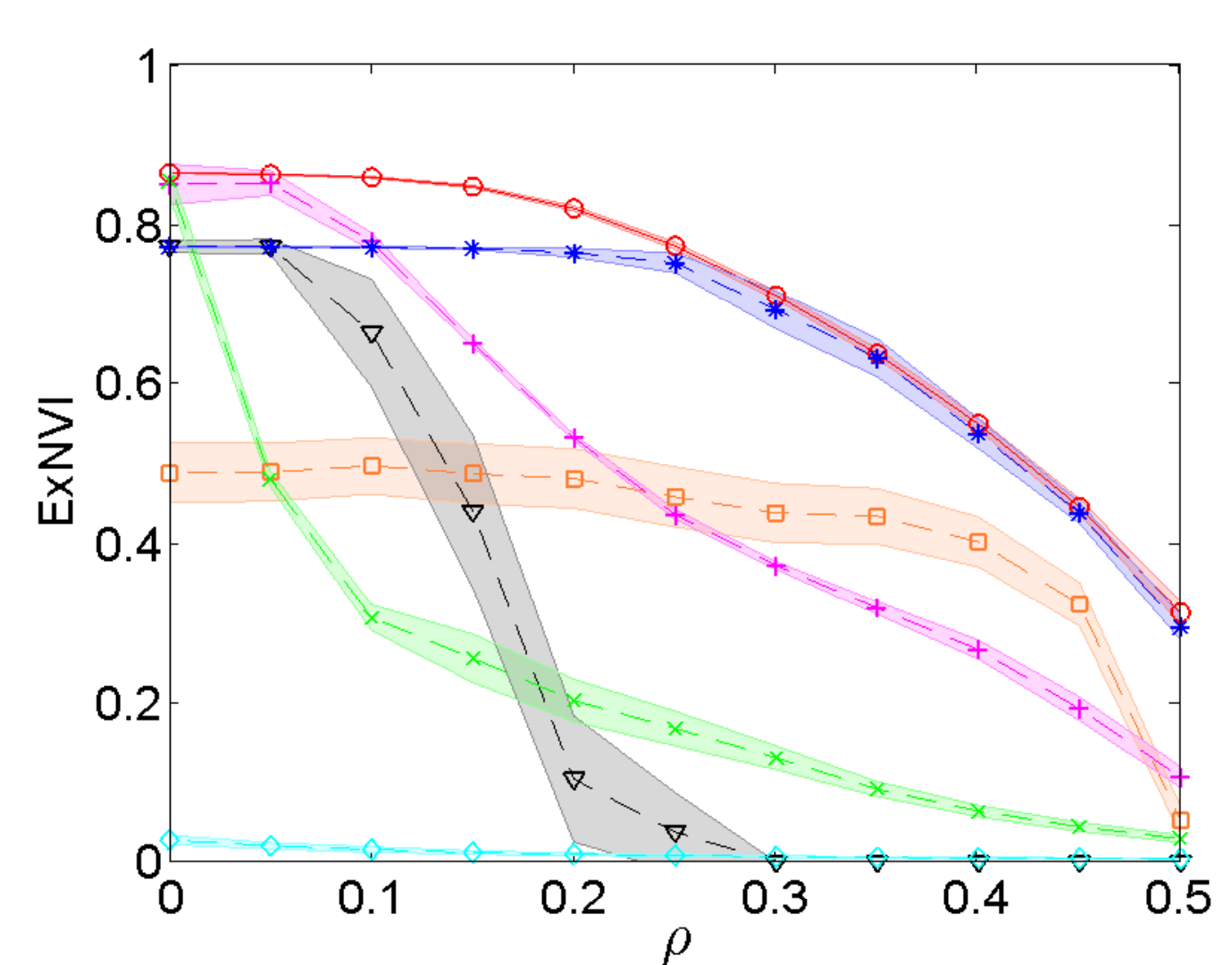}{0.4\textwidth}{(f) A, $d=40$, no hub node}
\twoImages{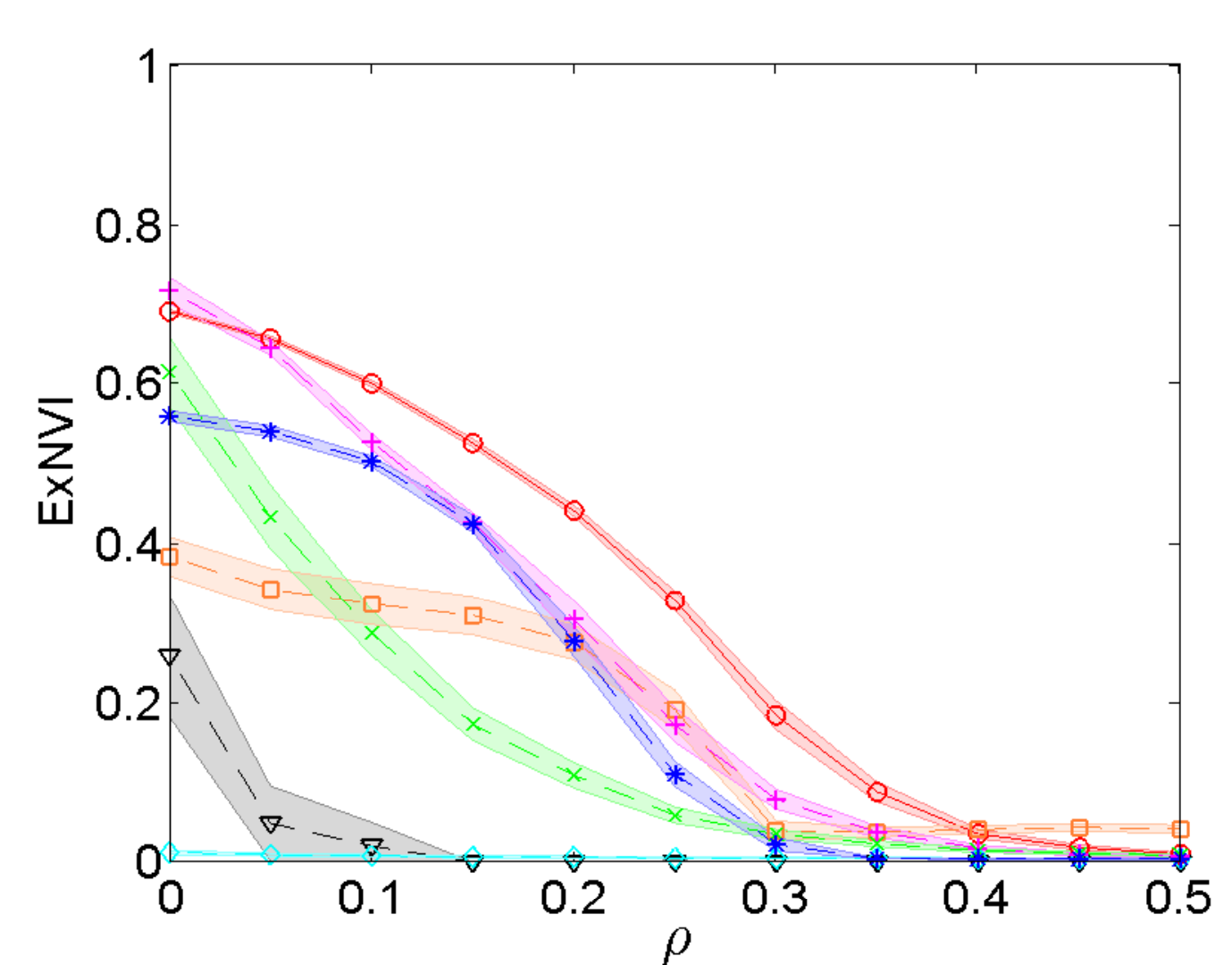}{0.4\textwidth}{(g) B, $d=20$, no hub node}  {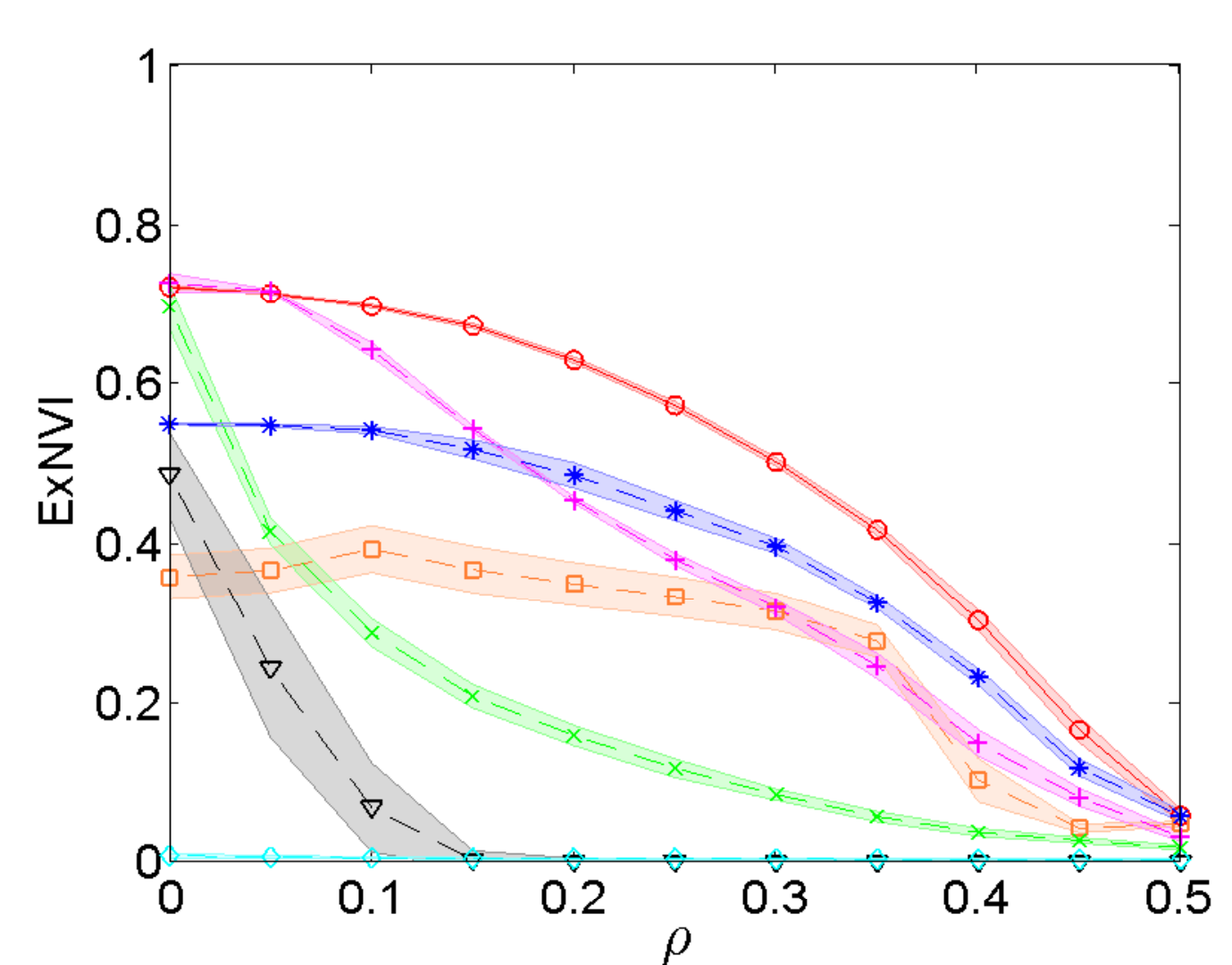}{0.4\textwidth}{(h) B, $d=40$, no hub node}
\vspace{-1em}
\caption{A: $(\pi^{(1)}, \pi^{(2)}, \pi^{(3)}) = (0.3, 0.03, 0.03)$; B:$(\pi^{(1)}, \pi^{(2)}, \pi^{(3)}) = (0.25, 0.07, 0.04)$}
\end{figure}

%% file: data_applications.tex

The ego network datasets \citep{Jure_ego} contain more than 1000 ego-networks from Facebook, Twitter and GooglePlus. In an ego network, all the nodes are friends of one central user, and the friendship groups or circles (depending on the platform) set by this user can be used as ground truth communities.  This dataset was introduced by \citet{Jure_ego}, who also proposed an algorithm for overlapping community detection, which we will refer to as ML.  We did not include this method in simulation studies because it uses additional node features which all other algorithms under comparison do not; however, we include it in comparisons in this section.  
Before comparing the methods, we carried out some pre-processing to make sure the test cases do in fact have a substantial community structure.   First, we ``cleaned'' each network by (1) dropping nodes that are not assigned to any community; (2) dropping isolated nodes; (3) dropping communities whose pure nodes are less than 10\% of the network size. 
Note that step (3) is done iteratively, i.e., after dropping the smallest community that does not meet this criterion, we inspect all remaining communities again and continue until either all communities meet the criterion or only one community remains.  After this process is complete, we select cleaned networks that (a) contain at least 30 nodes; (b) have at least 2 communities; and (c) have Newman-Girvan modularities \citep{Newman&Girvan2004} on the true communities of no less than 0.05, indicating some assortative community structure is present.  
These three rules
eliminated 19, 45 and 28 networks respectively of the 132 GooglePlus networks,
455, 236 and 99 networks respectively out of 973 Twitter networks, and (b) eliminated 3 out of 10 Facebook networks.  The remaining 40 GooglePlus networks, 183 Twitter networks, and 7 Facebook networks were used in all comparisons,  using exNVI to measure performance. 

To get a better sense of what the different social networks look like and how different characteristics potentially affect performance, we report the following summary statistics for each network:
(1) density $\sum_{ij}A_{ij}/(n(n-1))$, i.e., the overall edge probability;
(2) average node degree $d$;
(3) the coefficient of variation of node degrees (the standard deviation divided by the mean) $\sigma_d/d$, which measures the amount of heterogeneity in the node degrees; 
(4) the proportion of overlapping nodes $r_o$;
(5) Newman-Girvan modularity.  Even though modularity was defined for non-overlapping communities, it still reflects the strength of the community structure in the networks in this dataset, which only have a modest amount of overlaps.  
We report the means and standard deviations of these measures for each of the social networks in Table \ref{table:data-summary}.
Note that Facebook and Gplus networks tend to be larger than Twitter networks, while Twitter networks tend to be denser, with more homogeneous degrees as reflected by $\sigma_d/d$, though their smaller size makes these measures less reliable.   

To compare methods, we report the average performance over each of the social platforms and the corresponding standard deviation in Table \ref{table:data-mean}.  We also report the mean pairwise difference between OCCAM and each of the other methods, along with its standard deviation in Table \ref{table:data-pairwise}.
\begin{table}[htbp!]
\begin{center} 
   \caption{Mean (SD) of summary statistics for ego-networks}
    \makebox[\linewidth]{
    \begin{tabular}{ccccccccc}\hline
     		& \#Networks	& $n$	& $K$	& Density 	& $d$	& $\sigma_d/d$	& $r_o$ 	& Modularity \bigstrut\\\hline
    Facebook 	& 7 		& 224	& 3.3	& 0.137	& 28 		& 0.644		& 0.030		& 0.418 \\
		& - 		& (221)	& (0.8)	& (0.046) 	& (29)	& (0.145)		& (0.021)	& (0.148) \\\hline
    Gplus	& 40		& 414	& 2.3	& 0.170	 & 53	& 1.035		& 0.057 	& 0.171 \\
		& - 		& (330)	& (0.5)& (0.109)  & (34)		& (0.471)		& (0.077)	& (0.109) \\\hline
    Twitter 	& 183 		& 62	& 2.8	& 0.264 & 15		& 0.595			& 0.036 	& 0.204 \\
		& - 		& (31)	& (0.9)	& (0.264) 	& (8)	& (0.148)		& (0.055)	& (0.119) \\\hline
    \end{tabular}
    }
\label{table:data-summary}
\end{center}
\end{table}   


\begin{table}[htbp!]
\begin{center}
  \caption{Mean (SD) of exNVI for all methods.}
    \makebox[\linewidth]{
    \begin{tabular}{ccccccccc}\hline
		& OCCAM		& OSLOM		& COPRA 	& NMF   	& BNMF	& BKN 		& OSBM	& ML \bigstrut\\\hline
Facebook	& \bf{0.576}	& 0.212		& 0.394 	& 0.314		& 0.500	& 0.474 	& 0.473		& 0.133	\\
		& (0.116)	& (0.068)	& (0.115)	& (0.079)	& (0.094)	& (0.107)	& (0.114)	& (0.033)	\\\hline
Gplus		& \bf{0.503}	& 0.126		& 0.114 	& 0.293		& 0.393	& 0.357 	& 0.333 	& 0.175\\
		& (0.038)	& (0.017)	& (0.036)	& (0.036)	& (0.046)	& (0.030)&  (0.039)	& (0.023)\\\hline
Twitter 	& \bf{0.451} 	& 0.208 	& 0.232 	& 0.212	& 0.437  	& 0.346 	& 0.348		& 0.200\\
		& (0.021)	& (0.012)	& (0.023)	& (0.013)	& (0.021)	& (0.017)	& (0.017)	& (0.010)\\\hline
    \end{tabular}
    }
\label{table:data-mean}
\end{center}
\end{table}

\begin{table}[htbp!]
\begin{center}
	\caption{Mean (SD) of pairwise differences in exNVI between OCCAM and other methods.}
	\begin{tabular}{cccccccc}\hline
		& vs OSLOM & vs COPRA & vs NMF & vs BNMF & vs BKN & vs OSBM & vs ML\\\hline
		Facebook & 0.363  & 0.182  & 0.261  & 0.075  & 0.101  & 0.102 & 0.443\\
				 & (0.093) & (0.082) & (0.071) & (0.072) & (0.053) & (0.032) & (0.134)\\\hline
		Gplus	 & 0.377 & 0.389 & 0.210 & 0.110 & 0.146 & 0.171 & 0.328\\
				 & (0.037) & (0.037) & (0.040) & (0.038) & (0.020) & (0.028) & (0.042)\\\hline
		Twitter  & 0.243 & 0.219 & 0.239 & 0.014 & 0.105 & 0.103 & 0.251\\
				 & (0.020) & (0.019) & (0.016) & (0.012) & (0.012) & (0.011) & (0.024)\\\hline
	\end{tabular}
\label{table:data-pairwise}	
\end{center}
\end{table}


As in simulation studies, we observe that OCCAM outperforms other methods.
Gplus networks on average have the most heterogeneous node degrees and thus are challenging for COPRA and OSBM, while OCCAM is relatively robust to node degree heterogeneity.
Further, Gplus networks tend to have higher proportions of overlapping nodes than Facebook networks; this creates difficulties for all methods.
Empirically, we also found that OSLOM and COPRA are prone to convergence to degenerate community assignments, assigning all nodes to one community.  NMF, BNMF and BKN often create substantial overlaps compared to other methods, likely because they do not allow connections between pure nodes from different communities.  
The results suggest that OCCAM works well when the overlap is not large even when modularity is relatively low, while other methods are more sensitive to modularity, which measures the strength of an assortative community structure.  On the other hand,  large overlaps between communities cause the performance of OCCAM to deteriorate, which is consistent with our theoretical results.
ML is not readily comparable to others since it uses both network information and node features when fitting the model, and one would expect it do to better since it makes use of more information;  however, using node features that are uncorrelated with the community structure can in fact worsen community detection, which may explain its poor performance on some of the networks.   


A fair comparison of computing times is difficult because the methods compared here are implemented in different languages.  Qualitatively, we can say that the most expensive part of OCCAM is the $K$-medians clustering, which involves gradient descent, and is about one order of magnitude slower than NMF.   The computational cost of OCCAM is comparable to that of BNMF, BKN and COPRA, and is at least two orders of magnitude less than that of OSLOM, OSBM and ML.

%% file: discussion.tex
This paper makes two major contributions, the model and the algorithm.  The model we proposed for overlapping communities, OCCAM, is identifiable, interpretable, and flexible; it addresses limitations of several earlier approaches by allowing continuous community membership, allowing for pure nodes from different communities to be connected, and accommodating heterogeneous node degrees.  Our goals in designing an algorithm to fit the model were scalability and of course accuracy, and therefore we made a number of modifications to spectral clustering to deal with the overlaps, most importantly replacing $K$-means with $K$-medians. Empirically we found the algorithm is a lot faster than most of its competitors, and it performs well on both synthetic and real networks.    We also showed estimation consistency under conditions that articulate the appropriate setting for our method -- the overlaps are not too large and the network is not too sparse (the latter being a general condition for all community detection consistency, and the former specific to our method).  

In addition to its many advantages, our method has a number of limitations.  The upper bound on the amount of overlap is a restriction, expressed by implicit condition \ref{condition_derivative}, which may not be easy to verify except in special cases. It is clear, however, that some limit on the amount of overlap is necessary for any model to be identifiable.  Like all other spectral clustering based methods, OCCAM works best when communities have roughly similar sizes;  this is implied by condition \ref{condition_derivative} which implicitly excludes communities of size $o(n/K)$ as $n$ and $K$ grow.  Further, our model only applies to assortative communities, in other words, requires the matrix of probabilities $B$ to be positive definite. This constraint seems to be unavoidable if the model is to be identifiable.  

Like the vast majority of existing community detection methods, we assume that the number of communities $K$ is given as input to the algorithm.  There has been some very recent work on choosing $K$ by hypothesis testing \citep{Bickel&Sarkar2013hypothesis} or a BIC-type criterion \citep{Saldana.et.al2014howmany} for the non-overlapping case;  testing these methods and adapting them to the overlapping case is a topic for future work which is outside the scope of this manuscript but is an interesting topic.   Another interesting and difficult challenge is detecting communities in the presence of ``outliers'' that do not belong to any community, considered by \citet{Zhao.et.al.2011} and \citet{Cai2014}.  Our algorithm may be able to do this with additional regularization.   Finally, incorporating node features when they are available into overlapping community detection is another challenging task for future, since the features may introduce both additional useful information and additional noise.

%% file: appendix.tex
\subsection{Proof of identifiability}

\begin{proof}[Proof of Theorem \ref{theorem_identifiability}]

We start with stating a Lemma of \citet{tang2013universally}: 
\begin{lemma}[Lemma A.1 of \citet{tang2013universally}]\label{lemma_Tang_et_al}
	Let $\Y_1, \Y_2\in \mathbb{R}^{n\times d}$, $d<n$, be full rank matrices and $\G_1=\Y_1\Y_1^T$, $\G_2=\Y_2\Y_2^T$. Then there exists an orthonormal $\Oh$ such that
	\begin{align}
		\|\Y_1\Oh - _2\|_F &\leq \frac{ \sqrt{d}\|\G_1 - \G_2\|(\sqrt{\|\G_1\|} + \sqrt{\|\G_2\|}) }{\lambda_{\min}(\G_2)}
	\end{align}
	where $\lambda_{\min}(\cdot)$ i sthe smallest positive eigenvalue.
\end{lemma}

Lemma \ref{lemma_Tang_et_al} immediately implies 
\begin{claim}\label{claim_orthonormal_transformation}
	For two full rank matrices $\eicH_1$, $\eicH_2\in \mathbb{R}^{n\times K}$ satisfying $\eicH_1 \eicH_1^T = \eicH_2 \eicH_2^T$, there exists an orthonormal matrix $\Oh_H$ such that $\eicH_1\Oh_H = \eicH_2$.
\end{claim}

Suppose parameters $(\alpha_{n,1}, \BoldTheta_1, \Z_1, \B_1)$ and $(\alpha_{n,2}, \BoldTheta_2, \Z_2, \B_2)$ generate the same $\W$. Then by Lemma \ref{claim_orthonormal_transformation}, there exists an orthonormal matrix $\Oh_{12}$ such that
  \begin{equation}
   \alpha_{n,1} \BoldTheta_1 \Z_1 \B_1^{1/2} \Oh_{12} = \alpha_{n,2} \BoldTheta_2 \Z_2 \B_2^{1/2} \label{orthonormal_equation}
  \end{equation}
We then show that the indices for ``pure'' rows in $\Z_1$ and $\Z_2$ match up. More precisely, for $1\leq k \leq K$, let ${\cal I}_k:=\{i: row_i(\Z_1)= \ee_k\}$. We show that $row_j(\Z_2)$, $j\in\mi_k$ are also pure nodes, i.e., there exists $k'$ such that $\{j: row_j(\Z_2) = \ee_{k'}\} = \mi_k$. It suffices to show that there exists $i\in \mi_k$ such that $row_i(\Z_2)$ is pure, then the claim follows from the fact that all rows in $\Z_2$ with indices in $I_k$ equal each other, since their counterparts in $\Z_1$ are equal.    We prove this by contradiction: if $\{row_i(\Z_2), i\in \mi_k\}$ are not pure nodes, then for any $i\in\mi_k$, there exists $\{i_1, \ldots, i_K\}\subset \{1,\ldots, n\} - \mi_k$ and $\omega_1, \ldots, \omega_K\geq 0$ such that
	\begin{equation}
		row_i(\Z_2) = \sum_{k=1}^K \omega_k  row_{i_k}(\Z_2)
	\end{equation}
By (\ref{orthonormal_equation}), this yields
	\begin{equation}
		row_i(\Z_1) = \sum_{k=1}^K \omega_k \frac{\alpha_{n,1}(\BoldTheta_1)_{i_k i_k}}{\alpha_{n,2}(\BoldTheta_2)_{i_k i_k}}  row_{i_k}(\Z_1)
	\end{equation}
i.e. the $i$th row of $\Z_1$ can be expressed as a non-negative linear combination of at most $K$ rows outside $\mi_k$, and thus $row_i(\Z_1)$ is not pure. Essentially we have shown the identifiability for all pure nodes.
To show identifiability for the rest, take one pure node from each community as representative, i.e., let $\tilde{\mi}:=\{j_1,\ldots,j_K\}$, where $j_k\in \mi_k$, $1\leq k \leq K$. Let $\Z_1^K$ be the submatrix induced by concatenating rows of $\Z_1$ with indices in $\tilde{\mi}$,  similarly define $\Z_2^K$, and let $\tilde{\BoldTheta}_1$ and $\tilde{\BoldTheta}_2$ be the corresponding submatrices of $\BoldTheta_1$ and $\BoldTheta_2$. Note $\Z_1^K$ and $\Z_2^K$ are both order $K$ permutations, which is an ambiguity allowed by our definition of identifiability, so we take $\Z_1^K=\Z_2^K=I$. By \eqref{orthonormal_equation}, 
	\begin{equation}
		\alpha_{n,1}\tilde{\BoldTheta}_1 \B_1^{1/2} \Oh_{12} = \alpha_{n,2}\tilde{\BoldTheta}_2 \B_2^{1/2} \ .
\label{eq:8.5}
	\end{equation}
By  condition \ref{condition_iden_B}, both $\B_1^{1/2} \Oh_{12}$ and $\B_2^{1/2}$ have rows of norm 1, so $\alpha_{n,1}\cdot(\tilde{\BoldTheta}_1)_{kk} = \|row_k(\alpha_{n,1}\BoldTheta_1^K \B_1^{1/2} \Oh_{12})\|_2 = \|row_k(\alpha_{n,2}\BoldTheta_2^K \B_2^{1/2})\|_2 \allowbreak= \alpha_{n,1}\cdot(\tilde{\BoldTheta}_2)_{kk}$ and therefore $\alpha_{n,1}\tilde{\BoldTheta}_1 = \alpha_{n,2}\tilde{\BoldTheta}_2$.    Then from \eqref{eq:8.5} we have
	\begin{equation}
		\B_1^{1/2} \Oh_{12} = \B_2^{1/2} \label{B1_O=B2}
	\end{equation}
Thus $\B_1 = \B_1^{1/2} \Oh_{12} (\B_1^{1/2} \Oh_{12})^T = \B_2$, and \eqref{orthonormal_equation}  implies $\alpha_{n,1}\BoldTheta_1=\alpha_{n,2}\BoldTheta_2$ since all rows of $\Z_1$ and $\Z_2$ are normalized.   This in turn implies $\alpha_{n,1} = \alpha_{n,2}$ by condition \ref{condition_iden_theta} and thus $\BoldTheta_1=\BoldTheta_2$. Finally, plugging all of this back into \eqref{orthonormal_equation} we have $\Z_1=\Z_2$.
\end{proof}

\subsection{Proof of consistency}

{\bf Proof outline:} 
The proof of consistency of $\hat{\Z}$ follows the steps of the algorithm:  we first bound the difference between $\estimator$ and the row-normalized version of the true node positions $\X^*$ with high probability (Lemma \ref{lemma_1}); then bound the difference between $\hat{\eS}$ and the true community centers $\eS=\B^{1/2}$ (Lemma \ref{lemma_2}) with high probability; these combine to give a bound on the difference between $\hat{\Z}$ and $\Z$ (Theorem \ref{main theorem}).


\begin{lemma}\label{lemma_1}
	Assume conditions \ref{condition_theta_distribution}, \ref{condition_eigen_ZB} and \ref{condition_eigen_B} hold. When $\frac{\log n}{n\alpha_n}\to 0$ and $K=O(\log n)$, there exists a global constant $C_1$, such that with the choice $\tau_n=\frac{\alpha_n^{0.2}K^{1.5}}{n^{0.3}}$, for large enough $n$, we have
	\begin{align}
		\mathbb{P}\left( \frac{\|\hat{\X}_{\tau_n}^* \Oh_{\hat{X}} - \X^*\|_F}{\sqrt{n}} \leq \frac{C_1 K^{\frac{4}{5}}}{(n\alpha_n)^{\frac{1}{5}}} \right) \geq 1-P_1(n, \alpha_n, K) \label{equation_lemma_1}
	\end{align}
	where $P_1(n, \alpha_n, K)\to 0$ as $n\to\infty$, and $\Oh_{\hat{X}}$ is an orthonormal matrix depending on $\hat{X}$.
\end{lemma}

\begin{proof}[Proof of Lemma \ref{lemma_1}]
Define the population version of $\estimator$ as $\X_{\tau_n}^*\in\mathbb{R}^{n\times K}$, where $row_i(\X_{\tau_n}^*):=\frac{\X_{i\cdot}}{\|\X_{i\cdot}\|_2+\tau}$. We first bound $\|\estimator \Oh_{\hat{X}} - \X_{\tau_n}^*\|_F$ for a certain orthonormal matrix $\Oh_{\hat{X}}$ and then the bias term $\|\X_{\tau_n}^* - \X^*\|_F$. Then the triangular inequality gives \eqref{equation_lemma_1}.

We now bound $\|\estimator \Oh_{\hat{X}} - \X_{\tau_n}^*\|_F$. For any orthonormal matrix $\Oh$, 
\begin{align*}
	&\|row_i(\hat{\X}_{\tau_n}^*\Oh-\X_{\tau_n}^*)\|_2	= \|row_i(\hat{\X}_{\tau_n}^*)\Oh - row_i(\X_{\tau_n}^*)\|_2\\
	& = \Big\|  \frac{\hat{\X}_{i\cdot}\Oh}{\|\hat{\X}_{i\cdot}\|_2+\tau_n} - \frac{\X_{i\cdot}}{\|\X_{i\cdot}\|_2+\tau_n}  \Big\|_2 =\Big\|  \frac{\hat{\X}_{i\cdot}\Oh}{\|\hat{\X}_{i\cdot}\Oh\|_2+\tau_n} - \frac{\X_{i\cdot}}{\|\X_{i\cdot}\|_2+\tau_n}  \Big\|_2\\
	= &\frac{\| \hat{\X}_{i\cdot}\Oh(\|\X_{i\cdot}\|_2-\|\hat{\X}_{i\cdot}\Oh\|_2) + \|\hat{\X}_{i\cdot}\Oh\|_2(\hat{\X}_{i\cdot}\Oh-\X_{i\cdot}) + \tau_n(\hat{\X}_{i\cdot}\Oh - \X_{i\cdot}) \|_2}{(\|\hat{\X}_{i\cdot}\Oh\|_2+\tau_n)(\|\X_{i\cdot}\|_2+\tau_n)}\\
	\leq &\frac{(2\|\hat{\X}_{i\cdot}\Oh\|_2+\tau_n)\|\hat{\X}_{i\cdot}\Oh-\X_{i\cdot}\|_2}{(\|\hat{\X}_{i\cdot}\Oh\|_2+\tau_n)(\|\X_{i\cdot}\|_2+\tau_n)}	\leq \frac{2\|\hat{\X}_{i\cdot}\Oh - \X_{i\cdot}\|_2}{\|\X_{i\cdot}\|_2+\tau_n} \leq \frac{2\|\hat{\X}_{i\cdot}\Oh - \X_{i\cdot}\|_2}{\tau_n} \ . 
\end{align*}
Then
\begin{align*}
	\|\hat{\X}_{\tau_n}^*\Oh_{\hat{\X}} - \X_{\tau_n}^*\|_F & \leq \sqrt{ \sum_{i=1}^n\left(\frac{2}{\tau_n}\right)^2\|\hat{\X}_{i\cdot}\Oh_{\hat{\X}} - \X_{i\cdot}\|_2^2 } = \frac{2\|\hat{\X}\Oh_{\hat{\X}} - \X\|_F}{\tau_n} \ .
\end{align*}
By Lemma \ref{lemma_Tang_et_al}, there exists an orthonormal matrix $\Oh_{\hat{\X}}$, such that
	\begin{align}
	 \|\hat{\X}_{\tau_n}^*\Oh_{\hat{\X}} - \X_{\tau_n}^*\|_F &\leq \frac{2\sqrt{K}\|\hat{\X}\hat{\X}^T - \X\X^T\|\left( \sqrt{\|\hat{\X}\hat{\X}^T\|} + \sqrt{\|\X\X^T\|} \right)}{\tau_n \lambda_{\min}(\X\X^T)}\nonumber\\
	 &\leq \frac{2\sqrt{K}\|\hat{\X}\hat{\X}^T - \X\X^T\|\left( \sqrt{\|\hat{\X}\hat{\X}^T-\X\X^T\|} + 2\sqrt{\|\X\X^T\|} \right)}{\tau_n \lambda_{\min}(\X\X^T)}\nonumber\\
	 &=\frac{2\sqrt{K}\|\A-\W\|\left(\sqrt{\|\A-\W\|}+2\sqrt{\|\W\|}\right)}{\tau_n\lambda_{\min}(\W)} \label{Lemma_1_first_ineq}
	\end{align}
where $\|\cdot\|$ denotes the operator norm. We then bound each term on the RHS of (\ref{Lemma_1_first_ineq}). To bound $\|\A - \W\|$, we mostly follow \citet{tang2013universally}. Let $\U$ and $\U$ be $n\times K$ matrices of the leading $K$ eigenvectors of $\A$ and $\W$ respectively, and define $\proja:=\hat{\U}\hat{\U}^T$ and $\projw:=\U\U^T$, then $\W = \X\X^T = \projw \X\X^T \projw = \projw \W \projw$, and similarly $\A = \proja \A \proja$. We have
	\begin{align}
		\|\A-\W\| =&  \|\proja \A \proja - \projw \W \projw\|\nonumber\\
		\leq & \|\proja (\A-\W) \proja\| + \|(\proja-\projw) \W \proja\| + \|\proja \W (\proja - \projw)\nonumber\|\\
		      &+ \|(\proja-\projw) \W (\proja-\projw)\|\nonumber\\
		\leq &  \|\A-\W\| + 2\|\proja-\projw\|\|\W\| + \|\proja-\projw\|^2\|\W\|  \ .
\label{bound_hXhXT-XXT}
	\end{align}
	By Appendix A.1 of \citet{lei2013consistency}, we have
		\begin{equation}
			\|\proja - \projw\|\leq \|\proja - \projw\|_F\leq \frac{2\sqrt{2K}\|\A-\W\|}{\lambda_{\min}(\W)}  \ .
\label{bound_projections}
		\end{equation}
	By Theorem 5.2 of \citet{lei2013consistency}, when $\log n/(n\alpha_n)\to 0$ and $\theta_i$'s are uniformly bounded by a constant $M_\theta$, there exists constant $C_{r, M_\theta}$ depending on $r$, such that with probability $1-n^{-r}$
	\begin{equation}
		\|\A-\W\|\leq C_r\sqrt{n\alpha_n} \ .
	\end{equation}
	Since $M_\theta$ is a global constant in our setting, we write $C_r:=C_{r, M_\theta}$.  
	
	In order to bound $\|\estimator \Oh_{\hat{X}} - \X_{\tau_n}^*\|_F$, it remains to bound the maximum and minimum eigenvalues of $\W$. We will show that the eigenvalues of $(n\alpha_n)^{-1}\W$ converge to those of $\mathbb{E}[\theta_1^2 \Z_{1\cdot}^T \Z_{1\cdot}]\B$, which is strictly positive definite: for any $v\in\mathbb{R}^{K}$, 
    $$ v^T \mathbb{E}[\Z_{1\cdot}\Z_{1\cdot}^T]\geq \sum_{k=1}^K\mathbb{P}(1\in\mc_k)\cdot v^T \ee_k \ee_k^T v\geq 0 \ , $$ 
where $\mc_k$ denotes the set of nodes in community $k$ and $\ee_k$ denotes the vector the $k$th element equal to 1 and all others being 0. Equality holds only when all $v^T \ee_k \ee_k^T v = v_k^2 = 0$, i.e. $v=0$.
	\begin{claim}\label{claim_A2}
		Assume that $\theta_i>0$ for all $i$, and both $\Z$ and $B$ are full rank. Let $\lambda_0$ and $\lambda_1$ denote the smallest and largest eigenvalues of $\mathbb{E}[\theta_1^2 \Z_{1\cdot}^T \Z_{1\cdot}]B$.  Then 
		\begin{align}
		 \mathbb{P}\left( \Big|\frac{\lambda_{\max}(\W)}{n\alpha_n} - \lambda_1\Big| > \epsilon \right) &\leq 2K^2 \exp\left(-\frac{\frac{1}{2}n\epsilon^2}{M_\theta^4K^3+\frac{1}{3}M_\theta^2K\sqrt{K}\epsilon}\right) \label{bound_max_eig_value}\\
		 \mathbb{P}\left( \Big|\frac{\lambda_{\min}(\W)}{n\alpha_n} - \lambda_0\Big| > \epsilon \right) &\leq 2K^2 \exp\left(-\frac{\frac{1}{2}n\epsilon^2}{M_\theta^4K^3+\frac{1}{3}M_\theta^2K\sqrt{K}\epsilon}\right) \label{bound_min_eig_value}
		\end{align}
	\end{claim}
	\begin{proof}[Proof of Claim \ref{claim_A2}]
		For $k=1,\ldots, K$, let $\lambda_k$ denote the $k$th largest eigenvalue of $\W$, then
			\begin{align*}
				\lambda_k\left(\frac{\W}{n\alpha_n}\right) &= \lambda_k\left(\frac{{\BoldTheta} \Z \B \Z^T {\BoldTheta}}{n}\right) = \lambda_k\left( \frac{\B^{1/2} \Z^T {\BoldTheta}^2 \Z \B^{1/2}}{n} \right)\\
				&= \lambda_k\left(\frac{\Z^T {\BoldTheta}^2 \Z \B}{n}\right) = \lambda_k\left(\frac{1}{n}\sum_{i=1}^n\theta_i^2 \Z_{i\cdot}^T \Z_{i\cdot} \B\right)
			\end{align*}
		where the second equality is due to the fact that $\X\X^T$ and $\X^T \X$ share the same $K$ leading eigenvalues ($\X=\sqrt{\alpha_n}{\BoldTheta} \Z \B^{1/2}$). The third equality holds because $\B^{1/2}$ is full rank. To show \eqref{bound_min_eig_value},  it suffices to show that
			\begin{equation}
				\mathbb{P}\left( \Big\| \frac{1}{n}\sum_{i=1}^n\theta_i^2 \Z_{i\cdot}^T \Z_{i\cdot} \B - \mathbb{E}[\theta_1^2 \Z_{1\cdot}^T\Z_{1\cdot}\B] \Big\| > \epsilon \right) \leq 2\exp\left(-\frac{\frac{1}{2}n\epsilon^2}{M_\theta^4K^3+\frac{1}{3}M_\theta^2K\sqrt{K}\epsilon}\right) \label{bound_all_eig_values}
			\end{equation}
 For any $k, l\in\{1, \ldots, K\}$,  \{$\theta_i^2 (\Z_{i\cdot}^T \Z_{i\cdot} \B)_{kl}\}_i$ are an iid sequence uniformly bounded by $M_\theta^2\sqrt{K}$ with mean $\left(\mathbb{E}[\theta_i^2 (\Z_{i\cdot}^T \Z_{i\cdot} \B)]\right)_{kl}$.   By Bernstein's inequality, 
			\begin{equation*}
				\mathbb{P}\left( \Big|\left(\frac{1}{n}\sum_{i=1}^n\theta_i^2\Z_{i\cdot}^T\Z_{i\cdot}\B - \mathbb{E}[\theta_1^2\Z_{1\cdot}^T\Z_{1\cdot}\B]\right)_{kl}\Big| > \epsilon \right) \leq 2\exp\left(-\frac{\frac{1}{2}n\epsilon^2}{M_\theta^4K+\frac{1}{3}M_\theta^2\sqrt{K}\epsilon}\right) \ . 
			\end{equation*}
		By the union bound  and $\| A \| \leq \| A \|_{F}$, we have
			\begin{align*}
				& \mathbb{P}\left( \Big\| \frac{1}{n}\sum_{i=1}^n\theta_i^2\Z_{i\cdot}^T\Z_{i\cdot}\B - \mathbb{E}[\theta_1^2\Z_{1\cdot}^T\Z_{1\cdot}\B] \Big\| > K\epsilon \right)
			\leq 2K^2 \exp\left(-\frac{\frac{1}{2}n\epsilon^2}{M_\theta^4K+\frac{1}{3}M_\theta^2\sqrt{K}\epsilon}\right)  \ . 
			\end{align*}
		Replacing $\epsilon$ by $\epsilon/K$ completes the proof of Claim \ref{claim_A2}.
	\end{proof}
	
 We now return to the proof of Lemma \ref{lemma_1} and complete the bound on $\|\estimator \Oh_{\hat{X}} - \X_{\tau_n}^*\|_F$. Taking $\epsilon$ to be $\frac{\lambda_1}{2}$ and $\frac{\lambda_0}{2}$ respectively in (\ref{bound_max_eig_value}) and (\ref{bound_min_eig_value}), by Claim \ref{claim_A2}, $\|W\|\leq \frac{3}{2}n\alpha_n\lambda_1\leq \frac{3}{2}\mlone n\alpha_n K$ and $\lambda_{\min}(\W)\geq \frac{1}{2}n\alpha_n\lambda_0\geq \frac{1}{2}\mlzero n_\alpha$ hold with probability:
	\begin{align*}
	&1-2K^2 \exp\left(-\frac{\frac{1}{8}n\lambda_0^2}{M_\theta^4K^3+\frac{1}{6}M_\theta^2K\sqrt{K}\lambda_0}\right) + 2K^2 \exp\left(-\frac{\frac{1}{8}n\lambda_1^2}{M_\theta^4K^3+\frac{1}{6}M_\theta^2K\sqrt{K}\lambda_1}\right)\\ \geq & 1 - 4K^2\exp\left( -\frac{\frac{1}{8}n\mlzero^2}{M_\theta^4K^5 + \frac{1}{6}M_\theta^2K^{5/2}\mlzero} \right)
	\end{align*}
	Plugging this, together with (\ref{bound_projections}) and (\ref{bound_projections}), back into (\ref{bound_hXhXT-XXT}), we have
	\begin{align}
		\|\hat{\X}\hat{\X}^T - \X\X^T\| &\leq \|\A-\W\|\left(1+\frac{4\sqrt{2K}\|\W\|}{\lambda_{\min}(\W)} + \frac{8K\|\A-\W\|\|\W\|}{(\lambda_{\min}(\W))^2}\right)\nonumber\\
		&\leq C_r\sqrt{n\alpha_n}\left(1+\frac{12\sqrt{2K}\mlone}{\mlzero} + \frac{48K^2 C_r\mlone}{\mlzero^2 \sqrt{n\alpha_n}}\right) \label{final_bound_XhXhT_XXT}
	\end{align}
	with probability at least $1 - 4K^2\exp\left( -\frac{\frac{1}{8}n\mlzero^2}{M_\theta^4K^5 + \frac{1}{6}M_\theta^2K^{5/2}\mlzero} \right) - n^{-r}$. Then plugging (\ref{final_bound_XhXhT_XXT}) and Claim \ref{claim_A2} into \eqref{Lemma_1_first_ineq}, we have

	\begin{align}
			&\|\hat{\X}_{\tau_n}^* \Oh_{\hat{X}} - \X_{\tau_n}^*\|_F\nonumber\\
		\leq 	&\frac{2\sqrt{K}}{\tau_n}\frac{\|\hat{\X}\hat{\X}^T-\X\X^T\| \left( \sqrt{\|\hat{\X}\hat{\X}^T-\X\X^T\|} + 2\sqrt{\|\X\X^T\|} \right)}{\lambda_{\min}(\X\X^T)}\nonumber\\
		\leq 	&\frac{2\sqrt{K}}{\tau_n}\frac{ C_r\sqrt{n\alpha_n}\left( 1+\frac{12\sqrt{2K}K\mlone}{\mlzero} + \frac{48K^2 C_r\mlone}{\mlzero\sqrt{n\alpha_n}} \right) }{n\alpha_n \frac{\mlzero}{2K}}\nonumber\\
		\cdot	&\left( \left[ C_r\sqrt{n\alpha_n}\left( 1+\frac{12\sqrt{2K}K\mlone}{\mlzero} + \frac{48K^2 C_r\mlone}{\mlzero\sqrt{n\alpha_n}} \right) \right]^{\frac{1}{2}} + \sqrt{6\mlone n\alpha_n} \right)\nonumber\\
		= 	&\frac{4C_r\sqrt{K}}{\tau_n\mlzero}\left( 1+\frac{12\sqrt{2}\mlone K\sqrt{K}}{\mlzero} + \frac{48C_r\mlone K^2}{\mlzero^2\sqrt{n\alpha_n}} \right)\nonumber\\
		\cdot 	&\left( \left[ C_r\left( \frac{1}{\sqrt{n\alpha_n}} + \frac{12\sqrt{2}\mlone}{\mlzero}\frac{K\sqrt{K}}{\sqrt{n\alpha_n}} + \frac{48C_r\mlone}{\mlzero^2}\frac{K^2}{n\alpha_n} \right) \right]^{\frac{1}{2}} + \sqrt{6\mlone} \right)\label{final_bound_lemma_1}
	\end{align}

	By assumption $K=O(\log(n))$, we have $\frac{K^3}{n\alpha_n}\to 0$, thus for large enough $n$, the following inequalities that simplify \eqref{final_bound_lemma_1} hold:
	\begin{align*}
		\frac{(24-12\sqrt{2})\mlone K\sqrt{K}}{\mlzero}-1 &\geq \frac{48C_r \mlone K^2}{\mlzero^2\sqrt{n\alpha_n}}\\
		(3-\sqrt{6})\sqrt{\mlone} &\geq \left[ C_r\left( \frac{1}{\sqrt{n\alpha_n}} + \frac{12\sqrt{2}\mlone}{\mlzero}\frac{K\sqrt{K}}{\sqrt{n\alpha_n}} + \frac{48C_r \mlone}{\mlzero^2}\frac{K^2}{n\alpha_n} \right) \right]^{\frac{1}{2}} \ , 
	\end{align*}
 and	we have
	\begin{equation}
	 \textrm{RHS of (\ref{final_bound_lemma_1})} \leq \tilde{C}_r \cdot \frac{K^2}{\tau_n} \label{proof_lemma_1_part_1_final}
	\end{equation}
	where the constant $\tilde{C}_r:=\frac{288C_r \mlone^{3/2}}{\mlzero^2}$, which for simplicity we will continue to write as $C_r$. This completes the bound on $\|\estimator \Oh_{\hat{X}} - \X_{\tau_n}^*\|_F$.

The second part of the proof requires a bound on $\|\X_{\tau_n}^* - \X^*\|_F$.  From the defintion of $\X_{\tau_n}^*$, we can write 
	\begin{align}
		\|\X_{\tau_n}^* - \X^*\|_F^2 &= 
 \sum_{i=1}^n\left( \frac{\tau_n/\sqrt{\alpha_n}}{\|\X_{i\cdot}\|_2/\sqrt{\alpha_n}+\tau_n/\sqrt{\alpha_n}} \right)^2 \ . 
\label{def_bias}
	\end{align}
	Since $\frac{\|X_{\cdot}\|_2}{\sqrt{\alpha_n}} = \theta_i\|\Z_{i\cdot}\B^{1/2}\|_2 =\theta_i\sqrt{\Z_{i\cdot}\B \Z_{i\cdot}^T} \geq \theta_i\sqrt{\lambda_{\min}(\B)}\geq\theta_i\sqrt{m_B}>0$, by assumption, for $\epsilon\in(0, \epsilon_0)$, we have $\mathbb{P} \left( \frac{\|\X_{i\cdot}\|_2}{\sqrt{\alpha_n}}<\epsilon\sqrt{m_B} \right) \leq \mathbb{P}(\theta_i<\epsilon) \leq C_\theta \epsilon$. Therefore, for any $\epsilon\in (0, \epsilon_0)$, we have
	\begin{align}
		\mathbb{E}\left[\left(\frac{\tau_n/\sqrt{\alpha_n}}{\|\X_{i\cdot}\|_2/\sqrt{\alpha_n} + \tau_n/\sqrt{\alpha_n}}\right)^2\right] &\leq C_\theta\epsilon + (1-C_\theta\epsilon)\left(\frac{\tau_n/\sqrt{\alpha_n}}{\epsilon\sqrt{m_B}+\tau_n/\sqrt{\alpha_n}}\right)^2 \label{ineq_on_bias}
	\end{align}
	By assumption, $\tau_n/\sqrt{\alpha_n}\to 0$, so for large enough $n$ such that $\tau_n/\sqrt{\alpha_n}<\epsilon_0^{3/2}$, taking $\epsilon:=(\tau_n/\sqrt{n})^{2/3}<\epsilon_0$, we have
	\begin{align*}
		\textrm{LHS of (\ref{ineq_on_bias})} &\leq C_\theta(\tau_n/\sqrt{\alpha_n})^{2/3} + (1-C_\theta(\tau_n/\sqrt{\alpha_n})^{2/3})\left( \frac{(\tau_n/\sqrt{\alpha_n})^{1/3}}{m_B+(\tau_n/\sqrt{\alpha_n})^{1/3}} \right)^2\\
		&\leq \left( C_\theta + m_B^{-1} \right) (\tau_n/\sqrt{\alpha_n})^{2/3}
	\end{align*}
	Then for any $\delta>0$, we have
	\begin{align}
		&\mathbb{P}\left( \frac{\|\X_{\tau_n}^* - \X^*\|_F^2}{n} - (C_\theta + m_B^{-1})(\tau_n/\sqrt{\alpha_n})^{2/3} > \delta \right)\nonumber\\
		\leq &\mathbb{P}\left( \frac{\|\X_{\tau_n}^* - \X^*\|_F^2}{n} - \mathbb{E}\left[\frac{\|\X_{\tau_n}^* - \X^*\|_F^2}{n}\right] > \delta \right)\leq \exp\left(-\frac{\frac{1}{2}\delta^2n}{1+\frac{1}{3}\delta}\right)\label{proof_lemma_1_part_2}
	\end{align}
	where the second inequality is Bernstein's inequality plus the fact that each summand in the numerator of (\ref{def_bias}) is uniformly bounded by $1$ with an expectation bounded by $\left(C_\theta+m_B^{-1}\right)(\tau_n/\sqrt{\alpha_n})^{2/3}$.
	
We can now complete the proof of Lemma \ref{lemma_1}. Combining \eqref{proof_lemma_1_part_1_final} and \eqref{proof_lemma_1_part_2} yields
	\begin{align}
		&\mathbb{P}\left( \frac{\|\hat{\X}_{\tau_n}^* \Oh_{\hat{X}} - \X^*\|_F}{\sqrt{n}} \leq \frac{C_r K^2}{\tau_n\sqrt{n}} + \delta + (C_\theta+m_B^{-1})(\tau_n/\sqrt{\alpha_n})^{2/3} \right)\nonumber\\
		\geq &1-P_1(n, \alpha_n, K; r) \label{lemma_1_before_final}
	\end{align}
	The optimal $\tau_n$ that minimizes the RHS of the inequality inside the probability is $\tau_n=\frac{\alpha_n^{0.2}K^{1.5}}{n^{0.3}}$ -- here for simplicity we drop the constant factor in $\tau_n$, the effect of which we evaluated empirically in Section \ref{section_simulations}. Plugging this into (\ref{lemma_1_before_final}) and taking $\delta=K(n\alpha_n)^{-\frac{1}{5}}$ and denote $C_1:=\left(\frac{2}{3}(C_\theta+m_B^{-1})C_r^{\frac{2}{3}}\right)^{\frac{3}{5}}+1+\left(\frac{3}{2}C_r(C_\theta+m_B^{-1})^{\frac{3}{2}}\right)^{\frac{2}{5}}$ and $P_1(n, \alpha_n, K; r) := n^{-r} - 4\exp\left( -\frac{\frac{1}{8}\mlzero^2 n}{M_\theta^4K^5+\frac{1}{6}M_\theta^2\mlzero K^{\frac{5}{2}}} \right) \allowbreak - 2\exp\left( -\
	\frac{\frac{1}{2}K^{\frac{8}{5}}n^{\frac{3}{5}}\alpha_n^{-\frac{2}{5}}}{1+\frac{1}{3}K^{\frac{4}{3}}(n\alpha_n)^{\frac{1}{5}}} \right)$, 
	we obtain Lemma \ref{lemma_1}.  Note that since we are free to choose and fix $r$, we can drop the dependence on it from $P_1$, as we did in the statement of Lemma \ref{lemma_1}.
\end{proof}

The next step is to show the convergence of the estimated cluster centers $\hat{\eS}$ to the population cluster centers $\eS_\mf$.

\begin{lemma}\label{lemma_2}
	Recall that $\mathcal{F}$ denotes the popualtion distribution of the rows of $\X^*$ and let $\hat{\eS}\in\arg\min_S \lossfunc_n(\hat{\X}_{\tau_n}^*; \eS)$ and $S_{\mathcal{F}} \in\arg\min_S \lossfunc(\mathcal{F}; \eS)$. Assume that conditions \ref{condition_theta_distribution}, \ref{condition_eigen_ZB}, \ref{condition_eigen_B}
	 and \ref{condition_derivative} hold. 
	Then if $\frac{\log n}{n\alpha_n}\to 0$ and $K=O(\log n)$, for large enough $n$ 
	we have
	\begin{align}
		\mathbb{P}\left(D_H(\hat{\eS}\Oh_{\hat{X}}; \eS_\mf) \leq \frac{C_2 K^{\frac{9}{5}}}{(n\alpha_n)^{\frac{1}{5}}}\right) &\leq 1 - P_1(n, \alpha_n, K) - P_2(n, \alpha_n, K)
	\end{align}
	where $C_2$ is a global constant, $P_2(n, \alpha_n, K)\to 0$ as $n\to\infty$ and $D_H(\cdot, \cdot)$ is as defined in condition \ref{condition_derivative}.
\end{lemma} 

\begin{proof}[Proof of Lemma \ref{lemma_2}]
	Since the rows of $\hat{\X}_{\tau_n}^*$ and $\X^*$ have $l_2$ norms bounded by 1, the sample space of $\mathcal{F}$ is uniformly bounded in the unit $l_2$ ball.  Following the argument of \citet{pollard1981strong}, we show that all cluster centers estimated by $K$-medians fall in the $l_2$ ball centered at origin with radius $3$, which we denote as ${\cal R}$.  Otherwise, if there exists an estimated cluster center $\es$ outside ${\cal R}$, it is at least distance $2$ away from any point assigned to its cluster. Therefore, moving $\es$ to an arbitrary point inside the unit ball yields an improvement in the loss function since any two points inside the unit ball are at most distance $2$ away from each other. 

	We first show the uniform convergence of $\lossfunc_n(\estimator \Oh_{\hat{X}}; \eS)$ to $\lossfunc(\mf; \eS)$ and then show the optimum of $\lossfunc_n(\estimator \Oh_{\hat{X}}; \eS)$ is close to that of $\lossfunc(\mf; \eS)$. Let $\hat{\eS} \Oh_{\hat{X}}:=\arg\min_S \lossfunc_n(\estimator \Oh_{\hat{X}}; \eS)$. We start with showing that
		\begin{equation}
			\sup_{\eS\subset {\cal R}}|\lossfunc_n(\hat{\X}_{\tau_n}^*\Oh_{\hat{X}}; \eS) - \lossfunc_n(\X^*; \eS)| \leq \frac{\|\hat{\X}_{\tau_n}^*\Oh_{\hat{X}} - \X^*\|_F}{\sqrt{n}} \label{part_1} \ .
		\end{equation}
	To prove \eqref{part_1}, take any $\es\in {\cal R}$. For each $i$, let $\hat{\es}$ and $\es$ be (possibly identical) rows in $\eS$ that are closest to $(\hat{\X}_{\tau_n}^*\Oh_{\hat{X}})_{i\cdot}$ and $\X^*_{i\cdot}$ respectively in $l_2$ norm.  We have
		\begin{align*}
			\|\X^*_{i\cdot} - \es\|_2 - \|(\hat{\X}_{\tau_n}^*\Oh_{\hat{X}})_{i\cdot} - \hat{\es}\|_2 
			&\leq \|(\hat{\X}_{\tau_n}^*\Oh_{\hat{X}})_{i\cdot} - \X^*_{i\cdot}\|_2
		\end{align*}
	and similarly, $\|(\hat{\X}_{\tau_n}^*\Oh_{\hat{X}})_{i\cdot} - \hat{\es}\|_2 - \|\X^*_{i\cdot} - \es\|_2 \leq \|\X^*_{i\cdot} - (\hat{\X}_{\tau_n}^*\Oh_{\hat{X}})_{i\cdot}\|_2$. Thus $|\|(\hat{\X}_{\tau_n}^*\Oh_{\hat{X}})_{i\cdot} - \hat{\es}\|_2 - \|\X^*_{i\cdot} - \es\|_2| \leq \|(\hat{\X}_{\tau_n}^*\Oh_{\hat{X}})_{i\cdot} - \X^*_{i\cdot}\|_2$. Combining this inequalities for all rows, we have
		\begin{align}
		& |\lossfunc_n(\hat{\X}_{\tau_n}^*\Oh_{\hat{X}}; \eS) - \lossfunc_n(\X^*; \eS)| 
			= \Big| \frac{1}{n}\sum_{i=1}^n \left(\|(\hat{\X}_{\tau_n}^*\Oh_{\hat{X}})_{i\cdot} - \hat{\es}\|_2 - \|\X^*_{i\cdot} - \es\|_2\right) \Big|\nonumber\\
			& \leq \sqrt{ \frac{1}{n} \sum_{i=1}^n \|(\hat{\X}_{\tau_n}^*\Oh_{\hat{X}})_{i\cdot} - \X^*_{i\cdot}\|_2^2 } = \frac{1}{\sqrt{n}} \|\hat{\X}_{\tau_n}^*\Oh_{\hat{X}} - \X^*\|_F  \ .
\label{uniform_bound_XhatO_Xstar}
		\end{align}
	Then since that (\ref{uniform_bound_XhatO_Xstar}) holds for any $\eS$, the uniform bound (\ref{part_1}) follows.
	
	For simplicity, we introduce the notation ``$\eS\subset {\cal R}$'', by which we mean that the rows of a matrix $\eS$ belong to the set ${\cal R}$. We now derive the bound for $\sup_{\eS\subset {\cal R}}|\lossfunc_n(\X^*; \eS) - \lossfunc(\mf; \eS)|$, which, without taking the supremum, is easily bounded by Bernstein's inequality. To tackle the uniform bound, we employ an $\epsilon$-net (see, for example, \citet{Haussler:1986:ESR:10515.10522}).   There exists an $\epsilon$-net ${\cal R}_\epsilon$, with size $|{\cal R}_\epsilon|\leq C_{\cal R}\frac{K}{\epsilon} \log\frac{K}{\epsilon}$, where $C_{\cal R}$ is a global constant. For any $\tilde{\eS}\subset {\cal R}_\epsilon$, $\tilde{\eS}\in\mathbb{R}^{K\times K}$, notice that 
$\min_{1\leq k \leq K}\|\X^*_{i\cdot}-\tilde{\eS}_{k\cdot}\|_2$ is a random variable uniformly bounded by $6$ with expectation $\lossfunc(\mf;\tilde{\eS})$ for each $i$. Therefore, 
by Bernstein's inequality, for any $\delta>0$ we have
		\begin{align}
			\mathbb{P}(|\lossfunc_n(\X^*; \ts) - \lossfunc(\mf; \ts)|>\delta)	&\leq \exp\left( -\frac{\frac{1}{2}n\delta^2}{4R_M^2+\frac{2}{3}R_M\delta} \right) = \exp\left( -\frac{n\delta^2}{72+4\delta} \right)
		\end{align}
	The number of all such $\tilde{\eS}\subset {\cal R}_\epsilon$ is bounded by
		\begin{align*}
			\Big|\{\ts: \ts\subset {\cal R}_\epsilon\}\Big| &= \begin{pmatrix}
									C_{\cal R} \frac{K}{\epsilon} \log\frac{K}{\epsilon}\\
									K
			                                            \end{pmatrix} \leq \left( C_{\cal R}\frac{K}{\epsilon} \log\frac{K}{\epsilon} \right)^K
		\end{align*}
	By the union bound, we have
		\begin{equation}
			\mathbb{P}\left( \sup_{\ts\in {\cal R}_\epsilon} \Big| \lossfunc(\X^*; \ts) - \lossfunc(\mf; \ts) \Big|>\delta \right) < \left( C_{\cal R}\frac{K}{\epsilon} \log\frac{K}{\epsilon} \right)^K     \exp\left( -\frac{n\delta^2}{72+4\delta} \right)
\label{part_2_1}
		\end{equation}
	
	The above shows the uniform convergence of the loss functions for $\tilde{\eS}$ from the $\epsilon$-net ${\cal R}_\epsilon$. We then expand it to the uniform convergence of all $\eS\subset {\cal R}$. For any $\eS\subset {\cal R}$, there exists $\ts\subset {\cal R}_\epsilon$, such that both $\lossfunc_n(\cdot; \eS)$ and $\lossfunc(\cdot; \eS)$ can be well approximated by $\lossfunc_n(\cdot; \ts)$ and $\lossfunc(\cdot; \ts)$ respectively. To emphasize the dependence of $\tilde{\eS}$ on $\eS$, we write $\ts=\ts(\eS)$. Formally, we now prove the following.
		\begin{align}
			\sup_{\eS\subset {\cal R}}|\lossfunc_n(\X^*; \eS) - \lossfunc_n(\X^*; \ts(\eS))| &< \epsilon		\label{S_ts_Xstar}\\
			\sup_{\eS\subset {\cal R}}|\lossfunc(\mf; \eS) - \lossfunc(\mf; \ts(\eS))| &< \epsilon	\label{S_ts_FXstar}
		\end{align}
	To prove (\ref{S_ts_Xstar}) and (\ref{S_ts_FXstar}), for any $\eS\subset {\cal R}$, let $\ts\subset {\cal R}_\epsilon$ be a matrix formed by concatenating the points in ${\cal R}_\epsilon$ that best approximate the rows in $\eS$. Notice that $\ts$ formed such way may contain less than $K$ rows. In this case, we arbitrarily pick points in ${\cal R}_\epsilon$ to enlarge $\ts$ to $K$ rows. 
For any $x\in \mathbb{R}^K$, let $\es_0$ be the best approximation to $x$ among the rows of $\eS$ and $\tilde{\es}_0$ be the best approximation to $\es_0$ among the rows of $\ts$; let $\tilde{\es}_1$ be the best approximation to $x$ among the rows of $\ts$ and let $\es_1$ be the point among the rows of $\eS$ that is best approximated by $\tilde{\es}_1$. Since $\|\x-\es_0\|_2 \leq \|\x-\es_1\|_2 \leq \|\x-\es_1\|_2 + \|\es_1-\tilde{\es}_1\|_2 \leq \|\x-\tilde{\es}_1\|_2+\epsilon$, and similarly, $\|\x-\tilde{\es}_1\|_2 \leq \|\x-\tilde{\es}_0\|_2 \leq \|\x-\es_0\|_2+\epsilon$, we have
		\begin{equation}
			\Big| \min_{1\leq k \leq K}\|\x-\eS_{k\cdot}\|_2 - \min_{1\leq k \leq K}\|\x-\ts_{k\cdot}\|_2 \Big| = \Big| \|\x-\es_0\|_2 - \|\x-\tilde{\es}_1\|_2 \Big| \leq \epsilon \label{epsilon_point_control}
		\end{equation}
which implies  (\ref{S_ts_Xstar}) and (\ref{S_ts_FXstar}).
	
	Combining (\ref{part_1}), (\ref{part_2_1}), (\ref{S_ts_Xstar}) and (\ref{S_ts_FXstar}), we have shown that with probability $P_2(n, \epsilon, \delta):=1-\left( C_{\cal R}\frac{K+2}{\epsilon} \log\frac{K+2}{\epsilon} \right)^K \exp\left( -\frac{n\delta^2}{72+4\delta} \right)$,
		\begin{align}
			& \sup_{\eS\subset {\cal R}}\Big| \lossfunc(\hat{\X}_{\tau_n}^*\Oh_{\hat{X}}; \eS) - \lossfunc(\mf; \eS) \Big|\nonumber\\
		\leq 	& \sup_{\eS\subset {\cal R}}|\lossfunc(\hat{\X}_{\tau_n}^*\Oh_{\hat{X}}; \eS) - \lossfunc(\X^*; \eS)| + \sup_{\eS\subset {\cal R}}|\lossfunc(\X^*; \eS) - \lossfunc(\X^*; \ts(\eS))|\nonumber\\& + \sup_{\ts\subset {\cal R}_\epsilon} \Big| \lossfunc(\X^*; \ts) - \lossfunc(\mf; \ts) \Big| + \sup_{\eS\subset {\cal R}}|\lossfunc(\mf; \eS) - \lossfunc(\mf; \ts(\eS))|\nonumber\\
			&\leq \frac{\|\hat{\X}_{\tau_n}^*\Oh_{\hat{X}} - \X^*\|_F}{\sqrt{n}} + \delta + 2\epsilon \label{uniform_approx_FXhat_FFX}
		\end{align}
	Finally, we use (\ref{uniform_approx_FXhat_FFX}) to bound $D_H(\hat{\eS}, \eS_\mf)$.  Note that 
		\begin{align*}
		& 	 \lossfunc(\mf; \hat{\eS} \Oh_{\hat{X}}) - \lossfunc(\mf; \eS_\mf) \leq | \lossfunc(\mf; \hat{\eS} \Oh_{\hat{X}}) -  \lossfunc(\hat{\X}_{\tau_n}^*\Oh_{\hat{X}}; \hat{\eS} \Oh_{\hat{X}})|\\
			 &+ (\lossfunc(\hat{\X}_{\tau_n}^*\Oh_{\hat{X}}; \hat{\eS} \Oh_{\hat{X}}) - \lossfunc(\hat{\X}_{\tau_n}^*\Oh_{\hat{X}}; \eS_\mf))
			  + |\lossfunc(\hat{\X}_{\tau_n}^*\Oh_{\hat{X}}; \eS_\mf) - \lossfunc(\mf; \eS_\mf)|\\
			 & \leq 2 \sup_{\eS\subset {\cal R}}\Big| \lossfunc(\hat{\X}_{\tau_n}^*\Oh_{\hat{X}}; \eS) - \lossfunc(\mf; \eS) \Big| \ .
		\end{align*}
	Taking $\delta=\epsilon = \frac{K^{\frac{4}{5}}}{(n\alpha_n)^{\frac{1}{5}}}$, define $P_2(n, \alpha_n, K):=P_2(n, \epsilon, \delta)$ with plug-in values of $\delta$ and $\epsilon$.  
        To summarize, we have that with probability at least $1-P_1(n, \alpha_n, K; r) - P_2(n, \alpha_n, K)$, the following holds:
	\begin{align*}
		D_H(\hat{\eS}\Oh_{\hat{X}}, \eS_\mf) &\leq (MK^{-1})^{-1}(\lossfunc(\mf, \hat{\eS}\Oh_{\hat{X}}) - \lossfunc(\mf, \eS_\mf))\\
		&\leq 2K/M \sup_{\eS\subset {\cal R}}\Big| \lossfunc(\hat{\X}_{\tau_n}^*\Oh_{\hat{X}}; \eS) - \lossfunc(\mf; \eS) \Big|\\
		&\leq 2K/M \left( \frac{\|\hat{\X}_{\tau_n}^*\Oh_{\hat{X}} - \X^*\|_F}{\sqrt{n}} + \delta + 2\epsilon \right)\\
		&\leq  \frac{(C_1+3) K^{\frac{9}{5}}}{M(n\alpha_n)^{\frac{1}{5}}} =:\frac{C_2  K^{\frac{9}{5}}}{M(n\alpha_n)^{\frac{1}{5}}}
	\end{align*}
	where we let $C_2:=(C_1+3)/M$. This concludes the proof of Lemma \ref{lemma_2}

\end{proof}


\begin{proof}[Proof of the main result (Theorem \ref{main theorem})]
	Without loss of generality, we assume that the rows of $\hat{\eS}$ and $\eS_\mf$ are aligned in the sense that $\|\hat{\eS}\Oh_{\hat{X}}-\eS_\mf\|_2=D_H(\hat{\eS}\Oh_{\hat{X}}, \eS_\mf)$. We denote the unnormalized projection coefficients of $\estimator$ onto $\hat{\eS}$ by $\hat{\Y}$ and $\X^*$ onto $\eS_\mf=\B^{1/2}$ by $\Y$. We have:
		\begin{align}
			Y&=\X^*(\B^{1/2})^T\B^{-1}=\X^*\eS_\mf^T(\eS_\mf \eS_\mf^T)^{-1}\\
			\hat{Y}&=\estimator \hat{\eS}^T (\hat{\eS} \hat{\eS}^T)^{-1}=\estimator \Oh_{\hat{X}}(\hat{\eS} \Oh_{\hat{X}})^T (\hat{\eS} \Oh_{\hat{X}} (\hat{\eS} \Oh_{\hat{X}})^T)^{-1}
		\end{align}
	Recall that $\X_{i\cdot} = \Z_{i\cdot}\eS_\mf$, thus $\X^*_{i\cdot}=\frac{\Z_{i\cdot}\eS_\mf}{\|\Z_{i\cdot}\eS_\mf\|_2}$, and
		\begin{align*}
			\|\Y_{i\cdot}\|_2 &= \X^*_{i\cdot}\eS_\mf^T (\eS_\mf \eS_\mf^T)^{-1}= \frac{\|\Z_{i\cdot}\|_2}{\|\Z_{i\cdot}\eS_\mf\|_2} 
			= \frac{1}{\|\sum_{i=1}^K Z_{ik}(\eS_\mf)_{k\cdot}\|_2} \\ 
            & \geq \frac{1}{\sum_{k=1}^K |Z_{ik}|\|(\eS_\mf)_{k\cdot}\|_2}
			= \frac{1}{\sum_{k=1}^K |Z_{ik}|} \geq  \frac{1}{\|\Z_{i\cdot}\|_2\sqrt{K}} = \frac{1}{\sqrt{K}} \ . 
		\end{align*}
	The difference between the row-normalized projection coefficients $\Z$ and $\hat{\Z}$ can be bounded by the difference between $\Y$ and $\hat{\Y}$, since 
		\begin{align}
			\|\hat{\Z}_{i\cdot} - \Z_{i\cdot}\|_2 & = \Big\| \frac{\hat{Y}_{i\cdot}\|\Y_{i\cdot}\|_2 - \Y_{i\cdot}\|\hat{\Y}_{i\cdot}\|_2}{\|\hat{\Y}_{i\cdot}\|_2\|\Y_{i\cdot}\|_2} \Big\|_2   
			\leq \frac{2\|\hat{\Y}_{i\cdot} - \Y_{i\cdot}\|_2}{\|\Y_{i\cdot}\|_2} 
\label{bound_Z_by_Y} \ .  
		\end{align}
Then we have 
	\begin{align}
			&\frac{\|\hat{\Z}-\Z\|_2}{\sqrt{n}} \leq \frac{2\|\hat{\Y} - \Y\|_F\sqrt{K}}{\sqrt{n}} 
		= 	2  \sqrt{\frac{K}{n}} \| \hat{\X}_{\tau_n}^*\Oh_{\hat{X}} (\hat{\eS} \Oh_{\hat{X}})^T(\hat{\eS} \Oh_{\hat{X}}(\hat{\eS} \Oh_{\hat{X}})^T)^{-1} - \X^* \eS_\mf^T (\eS_\mf \eS_\mf^T)^{-1} \|_F\nonumber\\
		\leq 	&2  \sqrt{\frac{K}{n}}  \Big\{  \|\hat{\X}_{\tau_n}^*\Oh_{\hat{X}}(\hat{\eS} \Oh_{\hat{X}})^T\left( \hat{\eS} \Oh_{\hat{X}}(\hat{\eS} \Oh_{\hat{X}})^T)^{-1} - ( \eS_\mf \eS_\mf^T )^{-1} \right)\|_F + \|\hat{\X}_{\tau_n}^*\Oh_{\hat{X}}( (\hat{\eS} \Oh_{\hat{X}})^T - \eS_\mf^T ) (\eS_\mf \eS_\mf^T)^{-1} \|_F \nonumber\\
		& + \|\left( \hat{\X}_{\tau_n}^*\Oh_{\hat{X}} - \X^* \right)\eS_\mf^T (\eS_\mf \eS_\mf^T)^{-1}\|_F \Big\}=: I_1 + I_2 + I_3 \ , \label{main_theorem_before_final}
	\end{align}
	where
		\begin{align*}
			I_1 &:= 2  \sqrt{\frac{K}{n}}  \|\hat{\X}_{\tau_n}^*\Oh_{\hat{X}}(\hat{\eS} \Oh_{\hat{X}})^T\left( (\hat{\eS} \Oh_{\hat{X}}(\hat{\eS} \Oh_{\hat{X}})^T)^{-1} - ( \eS_\mf \eS_\mf^T )^{-1} \right)\|_F 
			\leq 2  \sqrt{\frac{K}{n}}  \left( \|\X^*\|_F+\|\hat{\X}_{\tau_n}^*\Oh_{\hat{X}} - \X^*\|_F \right)    \\
  & \cdot \left( \|\eS_\mf\|_F + \|\hat{\eS} \Oh_{\hat{X}} - \eS_\mf\|_F \right)
		\|(\hat{\eS} \Oh_{\hat{X}}(\hat{\eS} \Oh_{\hat{X}})^T)^{-1} - (\eS_\mf \eS_\mf^T)^{-1}\|_F \ , \\
			I_2 &:= 2  \sqrt{\frac{K}{n}} \|\hat{\X}_{\tau_n}^*\Oh_{\hat{X}}( (\hat{\eS} \Oh_{\hat{X}})^T - \eS_\mf^T ) (\eS_\mf \eS_\mf^T)^{-1} \|_F\\
			&\leq 2  \sqrt{\frac{K}{n}} \left( \|\X^*\|_F+\|\hat{\X}_{\tau_n}^*\Oh_{\hat{X}} - \X^*\|_F \right) \|\hat{\eS} \Oh_{\hat{X}} - \eS_\mf\|_F \|(\eS_\mf \eS_\mf^T)^{-1}\|_F\\
			I_3 &:= 2  \sqrt{\frac{K}{n}} \|\left( \hat{\X}_{\tau_n}^*\Oh_{\hat{X}} - \X^* \right)\eS_\mf^T (\eS_\mf \eS_\mf^T)^{-1}\|_F\\
			&\leq 2 \sqrt{\frac{K}{n}} \|\hat{\X}_{\tau_n}^*\Oh_{\hat{X}} - \X^*\|_F\|\eS_\mf\|_F \|(\eS_\mf \eS_\mf^T)^{-1}\|_F
		\end{align*}
	The term $\|(\hat{\eS} \Oh_{\hat{X}}(\hat{\eS} \Oh_{\hat{X}})^T)^{-1} - (\eS_\mf \eS_\mf^T)^{-1}\|_F$ is bounded by the following claim.
	
	\begin{claim}\label{claim_A3}
		For two $K\times K$ matrices $\V_1$ and $\V_2$ such that $\|\V_2\|_F= \sqrt{K}$, $\|\V_1-\V_2\|_2\leq\epsilon$, and $\lambda_{\min}(\V_2)>0$, we have 
		\begin{align}
			\|(\V_1\V_1^T)^{-1} - (\V_2\V_2^T)^{-1}\|_F 	&\leq K^2\left(\lambda_{\min}(\V_2\V_2^T)-K(2+\epsilon)\epsilon\right)^{-1}\nonumber\\
						&\cdot(\lambda_{\min}(\V_2\V_2^T))^{-1}(2+\epsilon)\epsilon \label{ineq_lemma_A3}
		\end{align}
	\end{claim}
	\begin{proof}[Proof of Claim \ref{claim_A3}]
		First, 
		\begin{align*}
		&	\|\V_1\V_1^T - \V_2\V_2^T\|_F \leq \|\V_1(\V_1-\V_2)^T + (\V_1-\V_2)\V_2^T\|_F\\
			&\leq (\|\V_1\|_F+\|\V_2\|_F)\|\V_1-\V_2\|_F  \leq (\|\V_1-\V_2\|_F + 2\|\V_2\|_F)\|\V_1-\V_2\|_F\\
			&\leq (\sqrt{K}\epsilon + 2\sqrt{K}\epsilon)\sqrt{K}\epsilon = (2+\epsilon)K\epsilon \ . 
		\end{align*}
		Then we have
		\begin{align*}
				&\|(\V_1\V_1^T)^{-1}-(\V_2\V_2^T)^{-1}\|_F \leq  \|(\V_1\V_1^T)^{-1}\|_F\|(\V_2\V_2^T)^{-1}\|_F\|\V_1\V_1^T - \V_2\V_2^T\|_F\\
			\leq &\sqrt{K} (\lambda_{\min}(\V_1\V_1^T))^{-1} \sqrt{K} (\lambda_{\min}(\V_2\V_2^T))^{-1}\|\V_1\V_1^T - \V_2\V_2^T\|_F\\
			\leq &K \left( \lambda_{\min}(\V_2\V_2^T) - \|\V_1\V_1^T - \V_2\V_2^T\|_F \right)^{-1}(\lambda_{\min}(\V_2\V_2^T))^{-1}\|\V_1\V_1^T - \V_2\V_2^T\|_F\\
			\leq &K(\lambda_{\min}(\V_2\V_2^T)-K(2+\epsilon)\epsilon)^{-1}(\lambda_{\min}(\V_2\V_2^T))^{-1}K(2+\epsilon)\epsilon
		\end{align*}
	\end{proof}
	
	We are now ready to bound $I_1+I_2+I_3$. For large enough $n$ such that $\max\{C_1, C_2\} \frac{K^{\frac{13}{10}}}{(n\alpha_n)^{\frac{1}{5}}} < \frac{1}{2}$ and $\frac{C_2 K^{\frac{9}{5}}}{(n\alpha_n)^{\frac{1}{5}}} < \min\{1, \frac{m_B}{6K}\}$,  with probability at least $1 - P_1(n, \alpha_n, K; r) - P_2(n, \alpha_n, K)$, we have:
	\begin{align}
		I_1 &\leq 2\sqrt{K}\left( 1+\frac{C_1K^{\frac{4}{5}}}{(n\alpha_n)^{\frac{1}{5}}} \right) \left( \sqrt{K} + \frac{C_2 K^{\frac{9}{5}}}{(n\alpha_n)^{\frac{1}{5}}} \right)\nonumber\\
		      &\cdot K^2\left( m_B-K\left( 2+ \frac{C_2 K^{\frac{9}{5}}}{(n\alpha_n)^{\frac{1}{5}}}\right) \frac{C_2 K^{\frac{9}{5}}}{(n\alpha_n)^{\frac{1}{5}}} \right)^{-1}\nonumber \cdot m_B^{-1} \left( 2+ \frac{C_2 K^{\frac{9}{5}}}{(n\alpha_n)^{\frac{1}{5}}}\right) \frac{C_2 K^{\frac{9}{5}}}{(n\alpha_n)^{\frac{1}{5}}}\nonumber\\
		&\leq 2\sqrt{K} \cdot \frac{3}{2} \cdot \frac{3}{2}\sqrt{K} \cdot K^2 \cdot \frac{2}{m_B^2} \cdot \frac{3C_2 K^{\frac{9}{5}}}{(n\alpha_n)^{\frac{1}{5}}} = \frac{C_{I_1} K^{\frac{43}{10}}}{(n\alpha_n)^{\frac{1}{5}}} \label{final_I_1}
	\end{align}
	where $C_{I_1}:=\frac{27 C_2}{m_B^2}$, and
	\begin{align}
		I_2 &\leq 2\sqrt{K}\left( 1+\frac{C_1 K^{\frac{4}{5}}}{(n\alpha_n)^{\frac{1}{5}}} \right)\cdot \frac{C_2 K^{\frac{9}{5}}}{(n\alpha_n)^{\frac{1}{5}}} \frac{\sqrt{K}}{m_B} \leq 2\sqrt{K} \cdot \frac{3}{2} \cdot \frac{C_2 K^{\frac{9}{5}}}{(n\alpha_n)^{\frac{1}{5}}} \frac{\sqrt{K}}{m_B} = \frac{C_{I_2}K^{\frac{14}{5}}}{(n\alpha_n)^{\frac{1}{5}}} \label{final_I_2}
	\end{align}
	where $C_{I_2}:=\frac{3 C_2}{m_B}$, and
	\begin{align}
		I_3 &\leq 2\sqrt{K} \cdot \frac{C_1 K^{\frac{4}{5}}}{(n\alpha_n)^{\frac{1}{5}}} \cdot \sqrt{K} \cdot \frac{\sqrt{K}}{m_B} = \frac{C_{I_3} K^{\frac{23}{10}}}{(n\alpha_n)^{\frac{1}{5}}}  \ . \label{final_I_3}
	\end{align}
	where $C_{I_3}:=\frac{2 C_1}{m_B}$.
	Plugging (\ref{final_I_1}), (\ref{final_I_2}) and (\ref{final_I_3}) back to (\ref{main_theorem_before_final}), we have
	\begin{align}
		\frac{\|\hat{Z}-Z\|_F}{\sqrt{n}} & \leq (C_{I_1}+C_{I_2}+C_{I_3})\frac{K^{\frac{43}{10}}}{(n\alpha_n)^{\frac{1}{5}}} \leq C_3 (n^{1-\alpha_0}\alpha_n)^{-\frac{1}{5}} \label{gives_K_growth_rate}
	\end{align}
	with probability at least $1 - P(n, \alpha_n, K; r)$, where $P(n, \alpha_n, K; r) := P_1(n, \alpha_n, K; r) \allowbreak - P_2(n, \alpha_n, K)$, and $C_3:=\left(C_{I_1} + C_{I_2} + C_{I_3}\right)\Big/\left(\sup_n K^{\frac{43}{10}}n^{-\alpha_0}\right)$. This completes the proof.
	
\end{proof}